\def\year{2021}\relax
\documentclass[letterpaper]{article} 
\usepackage{aaai21}  
\usepackage{times}  
\usepackage{helvet} 
\usepackage{courier}  
\usepackage[hyphens]{url}  
\usepackage{graphicx} 
\usepackage{natbib}  
\usepackage{caption} 
\usepackage{bbm}
\usepackage{amsthm}
\usepackage{amsfonts,amssymb}
\usepackage{amsmath}
\newtheorem{theorem}{Theorem}
\newtheorem{lemma}{Lemma}
\newtheorem{definition}{Definition}
\newtheorem{assumption}{Assumption}
\newtheorem{fact}{Fact}
\newtheorem{discussion}{Discussion}
\newtheorem{example}{Example}
\newtheorem{remark}{Remark}
\newtheorem{proposition}{Proposition}
\newtheorem{corollary}{Corollary}
\usepackage{paralist}
\usepackage{algorithm}
\usepackage{algorithmic}
\usepackage{color}
\usepackage{subfigure}
\usepackage{hyperref}
\usepackage{url}
\usepackage{makecell}

\usepackage{booktabs}
\usepackage{caption}
\usepackage{colortbl}
\usepackage{xcolor}
\usepackage{multirow}
\usepackage{pifont}
\usepackage{bm}
\usepackage{tikz}
\usetikzlibrary{automata, positioning}

\usepackage{sidecap}
\usepackage{framed}
\usepackage{makecell}

\hypersetup{colorlinks=false}

\title{On Convergence of Gradient Expected Sarsa($\lambda$) }

\begin{document}
\author{
    Long Yang\textsuperscript{\rm 1},
   Gang Zheng\textsuperscript{\rm 1},
   Yu Zhang\textsuperscript{\rm 1},
    Qian Zheng\textsuperscript{\rm 2},
    Pengfei Li\textsuperscript{\rm 1}
    Gang Pan \textsuperscript{\rm 1}\thanks{Corresponding author.}\\
}
\affiliations{
    \textsuperscript{\rm 1}College of Computer Science and Technology, Zhejiang University, China.\\
    \textsuperscript{\rm 2}School of Electrical and Electronic Engineering, Nanyang Technological University,Singapore.\\


 \textsuperscript{1} {\{yanglong,gang\_zheng, hzzhangyu,pfl,gpan\}@zju.edu.cn};\quad \textsuperscript{2}{zhengqian@ntu.edu.sg}
}
\maketitle

\begin{abstract}
We study the convergence of $\mathtt{Expected~Sarsa}(\lambda)$ with linear function approximation.
We show that applying the off-line estimate (multi-step bootstrapping) to $\mathtt{Expected~Sarsa}(\lambda)$ is unstable for off-policy learning.
Furthermore, based on convex-concave saddle-point framework, we propose a convergent $\mathtt{Gradient~Expected~Sarsa}(\lambda)$ ($\mathtt{GES}(\lambda)$) algorithm.
The theoretical analysis shows that our $\mathtt{GES}(\lambda)$ converges to the optimal solution at a linear convergence rate, 
which is comparable to extensive existing state-of-the-art gradient temporal difference learning algorithms.
Furthermore, we develop a Lyapunov function technique to investigate how the step-size influences finite-time performance of $\mathtt{GES}(\lambda)$, such technique of Lyapunov function can be potentially generalized to other GTD algorithms.
Finally, we conduct experiments to verify the effectiveness of our $\mathtt{GES}(\lambda)$.
\end{abstract}

\section{Introduction}

Tabular $\mathtt{Expected~Sarsa}(\lambda)$ (with importance sampling) is one of the widely used methods for off-policy evaluation in reinforcement learning (RL), whose goal is to estimate the value function of a given target policy via the data that is generated from a behavior policy. 
Due to the high-dimensional state space, instead of tabular learning, a standard approach is to estimate the value function with a linear function \cite{sutton2018reinforcement}.
There is very little literature to study $\mathtt{Expected~Sarsa}(\lambda)$ with function approximation for off-policy learning.
To our best knowledge, Sutton and Barto \shortcite{sutton2018reinforcement} (section 12.9) firstly extend \emph{off-line} estimate (multi-step bootstrapping) to $\mathtt{Expected~Sarsa}(\lambda)$ with linear function approximation.

Unfortunately, as pointed out by Sutton and Barto \shortcite{sutton2018reinforcement} intuitively, their off-line approach may be unstable, i.e., 
their way to extend $\mathtt{Expected~Sarsa}(\lambda)$ with linear function approximation still lacks a provable convergence guarantee,
which is undesirable for RL.
It is critical to find the inherent essence of the above unstable appearance in $\mathtt{Expected~Sarsa}(\lambda)$, which not only makes a complement for existing off-policy learning methods but also provides some inspirations to design a stable algorithm.
Thus, extending $\mathtt{Expected~Sarsa}(\lambda)$ with linear function approximation for off-policy evaluation is a fundamental theoretical topic in RL, including: 
1) how to character the instability of off-line $\mathtt{Expected~Sarsa}(\lambda)$ with linear function approximation; 
2) how to derive a convergent algorithm; 
3) what convergence rate does $\mathtt{Expected~Sarsa}(\lambda)$ with linear function approximation can reach. 
We focus on these questions in this paper.

\textbf{Our Main Works} 
To address the above problems,
we propose Theorem \ref{positive-convergence}, which characters a sufficient and necessary condition that presents stability criteria of off-line update $\mathtt{Expected~Sarsa}(\lambda)$ with linear function approximation.
Theorem \ref{positive-convergence} requires the \emph{key} matrix (that has been defined in (\ref{def:A})) keeps the negative real components.
Unfortunately, due to the discrepancy between behavior policy and target policy, off-line $\mathtt{Expected~Sarsa}(\lambda)$ that is suggested by
Sutton and Barto \shortcite{sutton2018reinforcement} may not satisfy the condition appears in Theorem \ref{positive-convergence}, i.e., their scheme maybe unstable.
Then,
we use a classic counterexample to verify the above instability lies in the off-line update $\mathtt{Expected~Sarsa}(\lambda)$ with linear function approximation, see Example \ref{two-state-example}.

Furthermore, to get a stable algorithm, we derive an \emph{on-line} gradient $\mathtt{Expected~Sarsa}(\lambda)$ ($\mathtt{GES}(\lambda)$) algorithm.
Theorem \ref{linear-convergence} shows that the proposed $\mathtt{GES}(\lambda)$ learns the optimal solution at a linear convergence rate, which is comparable to extensive existing state-of-the-art gradient temporal difference learning algorithms.
Although \citet{xu2019two} prove $\mathtt{TDC}$ \cite{sutton2009fast_a} also converges at a linear convergence rate, 
they require a projection step that is unpractical in practice.
Besides, the fussy blockwise step-size appears in \cite{xu2019two} is more complicated than our step-size condition.
Additionally, the results of \cite{gal2018finite,lakshminarayanan2018linear} require an i.i.d assumption of function parameters, our proof removes this condition and achieve a better result than theirs.
A more detailed comparison and an adequate discussion are provided in Table \ref{table-1}.

Finally, inspired by \citet{srikant2019-colt-finite-time}, \citet{wang2019multistep}, \citet{gupta2019finite},
we develop a Lyapunov function technique to establish Theorem \ref{step-size-per}, which illustrates
the relationship between the finite-time performance of $\mathtt{GES}(\lambda)$ and step-size.
Result shows that the upper-bounded error consists of two different parts: the first error depends on both step-size and the size of samples, and such error decays geometrically as the samples increase; while the second error is only determined by the step-size and it is independent of samples.
Additionally, the technique of proving Theorem \ref{step-size-per} can be potentially generalized to other GTD algorithms.

\subsection{Notations}
We use $\text{Spec}(A)$ to denote the eigenvalues of the matrix $A\in\mathbb{C}^{p\times p}$, i.e., $\text{Spec}(A)=\{\lambda_1,\cdots,\lambda_p\}$, where $\lambda_i$ is the root of the characteristic equation $p(\lambda)=\text{det}(A-\lambda I)$.
We use $\mathbb{C}_{-}$ to denote the collection that contains the complex numbers with negative real components, i.e., 
\[
\mathbb{C}_{-}=:\{ c\in\mathbb{C};\text{Re}(c)<0\}.
\]
Let $\lambda_{\min}(A)$ and $\lambda_{\max}(A)$ be the minimum and maximum eigenvalue of the matrix $A$ correspondingly.
We use $\|A\|_{\text{op}}$ to denote the operator norm of matrix $A$; furthermore, if $A$ is a symmetric real matrix, then $\|A\|_{\text{op}}=\max_{1\leq i\leq p}\{|\lambda_{i}|\}.$
$\kappa(A)=\frac{\lambda_{\max}(A)}{\lambda_{\min}(A)}$ is the condition number of matrix $A$.
We use $A\succ 0$ to denote a positive definite matrix $A$.

For a function $f(x):\mathbb{R}^{p}\rightarrow \mathbb{R}$, let $\nabla^{2}f(x)$ denote its Hessian matrix, 
and its convex conjugate function 
$f^{\star}(y) : \mathbb{R}^{d}\rightarrow\mathbb{R}$ is defined as
$f^{\star}(y)=\sup_{x\in\mathbb{R}^{p}}\{y^{\top}x-f(x)\}.$ 
\begin{fact}[\cite{rockafellar1970convex,kakade2009duality}]
\label{fact1}
Let $f(\cdot)$ be $\alpha$-strongly convex and $\beta$-smooth, i.e., $f(\cdot)-\frac{\alpha}{2}\|\cdot\|_2^{2}$ is convex and $\|f(u)-f(v)\|\leq\beta\|u-v\|$.
If $0\leq\alpha\leq\beta$, then the following fact holds,

(I) $f^{\star}$ is $\frac{1}{\alpha}$-smooth and $\frac{1}{\beta}$-strongly convex.

(II) $\nabla f=(\nabla f^{\star})^{-1}$ and $\nabla f^{\star}=(\nabla f)^{-1}$.
\end{fact}

\section{Preliminary}

Reinforcement learning (RL) is formalized as Markov decision processes (MDP) which considers the tuple $\mathcal{M}=(\mathcal{S},\mathcal{A}, P,R,\gamma)$; $\mathcal{S}$ is the space of states, $\mathcal{A}$ is the space of actions;
${P}: \mathcal{S}\times\mathcal{A}\times\mathcal{S}\rightarrow[0,1]$,
$p_{s s^{'}}^a=P(S_{t}=s^{'}|S_{t-1}=s,A_{t-1}=a)$ is the probability of state transition from $s$ to $s^{'}$ under playing the action $a$;
$R(\cdot,\cdot): \mathcal{S}\times\mathcal{A}\rightarrow\mathbb{R}^{1}$ is the expected reward function;
$\gamma\in(0,1)$ is the discount factor.

The policy $\pi$ is a probability distribution on $\mathcal{S}\times\mathcal{A} $,
we use $\pi(a|s)$ to denote the probability of playing $a$ under the state $s$. 
Let $\{S_{t}, A_{t}, R_{t+1}\}_{t\ge 0}$ be generated by a given policy $\pi$, 
its \emph{state-action value function}
$
q^{\pi}(s,a) = \mathbb{E}_{\pi}[\sum_{t=0}^{\infty}\gamma^{t}R_{t+1}|S_{0} = s,A_{0}=a],
$
where $\mathbb{E}_{\pi}[\cdot|\cdot]$ is the conditional expectation on the actions selected according to $\pi$.
Let $\mathcal{B}^{\pi}: \mathbb{R}^{|\mathcal{S}|\times |\mathcal{A}|}\rightarrow \mathbb{R}^{|\mathcal{S}|\times |\mathcal{A}|}$ be the \emph{Bellman operator}:
 \begin{flalign}
 \label{bellman-op}
 \mathcal{B}^{\pi}:  q\mapsto R^{\pi}+\gamma P^{\pi}q,
 \end{flalign}
 where $P^{\pi}$$\in\mathbb{R}^{|\mathcal{S}| \times |\mathcal{S}|}$, $R^{\pi}$$\in\mathbb{R}^{|\mathcal{S}|\times|\mathcal{A}|}$, their corresponding elements are:
 $
 [P^{\pi}]_{s,s^{'}}= \sum_{a \in \mathcal{A}}\pi(a|s)p^{a}_{ss^{'}},[R^{\pi}]_{s,a}=R(s,a).
$
It is well-known that $q^{\pi}$ is the unique fixed point of $\mathcal{B}^{\pi}$, i.e., $\mathcal{B}^{\pi}q^{\pi}=q^{\pi}$, which is known as Bellman equation.

\textbf{Off-Policy Evaluation}
Let's consider the following trajectory $\mathcal{T}$ generated by a policy $\mu$:
\[\mathcal{T}=\{S_{0}=s,A_{0}=a,\cdots,S_{t},A_{t},R_{t+1},\cdots\},\]
where $A_{t}\sim\mu(\cdot|S_{t})$,$S_{t+1}\sim P(\cdot|S_{t},A_{t})$.
In RL, the task of off-policy evaluation is to estimate the value function of the target policy $\pi$ via the data that is generated by an another policy $\mu$ (that is called \emph{behavior policy}), where $\mu\ne\pi$.
\begin{assumption}
\label{ass:ergodicity}
The Markov chain induced by behavior policy $\mu$ is \emph{\textbf{ergodic}}, i.e., there exists a stationary distribution $\xi(\cdot,\cdot)$ over $\mathcal{S}\times\mathcal{A}$: for $\forall (S_{0},A_0)\in\mathcal{S}\times\mathcal{A}$, 
\begin{flalign}
\frac{1}{n}\sum_{k=1}^{n} P(S_{k}= s,A_{k}=a |S_{0},A_0)\overset{n\rightarrow\infty}\rightarrow \xi(s,a)>0.
\end{flalign}
\end{assumption}
The ergodicity of behavior policy $\mu$ is a standard assumption in off-policy learning \cite{bertsekas2012dynamic}, and it implies each-action pair can be visited under this behavior policy $\mu$.
In this paper, 
we use $\Xi$ to denote a diagonal matrix whose diagonal element is $\xi(s,a)$, i.e., \[\Xi=\text{diag}\{\cdots,\xi(s,a),\cdots\}.\]

\textbf{Temporal Difference (TD) Learning}
TD learning updates value function as follows, $\forall ~t\ge0$,
\begin{flalign}
\label{td-learning}
Q(S_{t},A_{t})\leftarrow Q(S_{t},A_{t})+\alpha_t\delta_{t},
\end{flalign}
where $Q(\cdot,\cdot)$ is an estimate of state-action value function, $\alpha_t$ is step-size and $\delta_{t}$ is TD error. 
Let $Q_{t}=Q(S_{t},A_{t})$,
if $\delta_{t}$ is expected TD error:
\begin{flalign}
\label{def:es-delta}
\delta_{t}^{\text{ES}}=R_{t+1}+\gamma\sum_{a\in\mathcal{A}}\pi(a|S_{t+1})Q(S_{t+1},a)- Q_{t},
\end{flalign}
then update (\ref{td-learning}) is $\mathtt{Expected~Sarsa}$. 

\textbf{Expected Sarsa($\lambda$)} \citet{sutton2018reinforcement} \footnote{It is noteworthy that the $\lambda$-return version of $\mathtt{Expected~Sarsa}$ appears in section 12.9 of \cite{sutton2018reinforcement} is limited in the case of linear function case, Eq.(\ref{G_t}) extends it to be a general case.} 
propose a multi-step TD learning that extends $\mathtt{Expected~Sarsa}$ to $\lambda$-return version:
for each $t\ge0$,
\begin{flalign}
\label{G_t}
G_{t}^{\lambda}=Q_{t}+\sum_{k=t}^{\infty}(\gamma\lambda)^{k-t}\bigg(\prod_{i=t+1}^{k}\frac{\pi(A_{i}|S_{i})}{\mu(A_{i}|S_{i})}\bigg)\delta^{\text{ES}}_{k},
\end{flalign}
where $\delta^{\text{ES}}_{k}$ is expected TD error. For the convenience, 
we set $\prod_{i=t}^{k}\rho_i=\prod_{i=t}^{k}\frac{\pi(A_{i}|S_{i})}{\mu(A_{i}|S_{i})}=\rho_{t:k}$, and $\rho_{t:t+1}=1$.

Finally, we introduce $\lambda$-operator $\mathcal{B}^{\pi}_{\lambda} $ that is a high level view of iteration (\ref{G_t}):
\begin{flalign}
\label{operator-es-1}
\mathcal{B}^{\pi}_{\lambda}: 
q&\mapsto q+\mathbb{E}_{\mu}[\sum_{k=t}^{\infty}(\lambda\gamma)^{k-t}\delta^{\text{ES}}_{k}\rho_{t+1:k}]\\
\label{operator-es}
&=q+(I-\lambda\gamma P^{\pi})^{-1}(\mathcal{B}^{\pi}q-q),
\end{flalign}
where $\mathcal{B}^{\pi}$ is defined in (\ref{bellman-op}).
For the limitation of space, we provide the derivation of (\ref{operator-es}) from (\ref{operator-es-1}) in Appendix A.2.

\subsection{Linear Function Approximation}
TD learning (\ref{td-learning}) requires a very huge table to store the estimate value function $Q(\cdot,\cdot)$ when $|\mathcal{S}|$ is very large, which implies tabular TD learning is considerably expensive for high-dimensional RL.
We often use a parametric function ${Q}_{\theta}(\cdot,\cdot)$ to approximate $q^{\pi}(s,a)$, i.e.,
\[
q^{\pi}(s,a)\approx\phi^{\top}(s,a)\theta=:{Q}_{\theta}(s,a),
\]
where $\theta\in\mathbb{R}^{p}$ is the parameter that needs to be learned,
$\phi(s,a)=(\varphi_{1}(s,a),\varphi_{2}(s,a),\cdots,\varphi_{p}(s,a))^{\top}$, 
and each
$\varphi_{i}:\mathcal{S}\times\mathcal{A}\rightarrow\mathbb{R}$.
Furthermore, ${Q}_{\theta}$ can be rewritten as a version of matrix
$
{Q}_{\theta}=\Phi\theta\approx q^{\pi},
$
where $\Phi$ is a matrix whose rows are the state-action feature vectors $\phi^{\top}(s,a)$.
In this paper, we mainly consider extending $\lambda$-return of $\mathtt{Expected~Sarsa}$ (\ref{G_t}) with linear function approximation.

\section{Off-Line Gradient Expected Sarsa($\lambda$)}
\label{Off-Line}

In this section, we use a counterexample to show the way to extend $\mathtt{Expected~Sarsa}$($\lambda$) with linear function approximation via off-line estimate is unstable for off-policy learning.

\textbf{Off-Line Update}
Sutton and Barto \shortcite{sutton2018reinforcement} provide a way to extend (\ref{G_t}) with linear function approximation as follows,
\begin{flalign}
\nonumber
    \theta_{t+1}&=\theta_{t}+\alpha_{t}(
    {G}^{\lambda}_{t}-Q_{\theta}(S_t,A_t)
    )\nabla Q_{\theta}(S_{t},A_t)\\
        \label{gradient-ES}
    &=\theta_{t}+\alpha_{t}\Big(\sum_{k=t}^{\infty}(\gamma\lambda)^{k-t}\delta_{k,\theta}^{\text{ES}}\rho_{t+1:k}\Big)\phi_t,
\end{flalign}
where $\alpha_{t}$ is step-size, $\phi_{t}=:\phi(S_{t},A_{t})$, and
\[\delta_{k,\theta}^{\text{ES}}=R_{k+1}+\gamma\theta^{\top}_{t}\mathbb{E}_{\pi}[\phi(S_{k+1},\cdot)]-\theta^{\top}_{t}\phi_k.\]
Furthermore, we can rewrite the expected parameter in (\ref{gradient-ES}):
\begin{flalign}
\label{expected-para-es}
\mathbb{E}_{\mu}[\theta_{t+1}]=\theta_{t} + \alpha_{t}(A\theta_t+b),
\end{flalign}
where
\begin{flalign}
\nonumber
A&=\mathbb{E}_{\mu}\Big(\sum_{k=t}^{\infty}(\gamma\lambda)^{k-t}\rho_{t+1:k}\phi_{t}\big(\gamma\mathbb{E}_{\pi}[\phi(S_{k},\cdot)]-\phi_{k}\big)^{\top}\Big)\\
\label{def:A}
&= \Phi^{\top}\Xi(I-\gamma\lambda P^{\pi})^{-1}(\gamma P^{\pi}-I)\Phi,\\
\nonumber
b&= \mathbb{E}_{\mu}\Big(\sum_{k=t}^{\infty}(\gamma\lambda)^{k-t}\rho_{t+1:k}\phi_{t}R_{k+1}\Big)\\
\label{def:b}
&=\Phi^{\top}\Xi(I-\gamma\lambda P^{\pi})^{-1}R.
\end{flalign}
If $\theta_t$ (\ref{expected-para-es}) converges to a certain point $\theta^{\star}$, then  $\theta^{\star}$ satisfies the following linear system:
\begin{flalign}
\label{optimal-sol}
A\theta^{\star}+b=0.
\end{flalign}
Such $\theta^{\star}$ satisfies (\ref{optimal-sol}) is called \emph{TD-fixed point}.

\textbf{Stability Criteria}
According to Sutton et al. \shortcite{sutton2016emphatic}; and Ghosh and Bellemare \shortcite{ghosh2020representations}, we formulate the stability of the iteration (\ref{expected-para-es}) as the next definition.
\begin{definition}[Stability]
Update rule (\ref{expected-para-es}) is stable if $\theta_{k}$ converges to the point $\theta^{\star}$ satisfies (\ref{optimal-sol}) for any initial $\theta_0$.
\end{definition}

\begin{theorem}[Stability Criteria]
\label{positive-convergence}
Under Assumption \ref{ass:ergodicity}, the off-line update (\ref{expected-para-es}) is stable if and only if the eigenvalues of the matrix A (\ref{def:A}) have negative real components, i.e.,
\begin{flalign}
\label{condition-stability}
\emph{Spec}(A)\subset\mathbb{C}_{-}.
\end{flalign}
\end{theorem}
We provode its proof in Appendix B.
Theorem \ref{positive-convergence} provides a sufficient and necessary condition (\ref{condition-stability}) that guarantees the stability  of iteration (\ref{gradient-ES}).
Unfortunately, for off-policy learning, the matrix $A$ (\ref{def:A}) can not guarantee the condition (\ref{condition-stability}) holds, which implies the iteration (\ref{gradient-ES}) may be divergent and unstable.
Now, we use the following example \cite{touati2018convergent} to illustrate the instability lies in the iteration (\ref{gradient-ES}).

\begin{example} 
\label{two-state-example}
For the MDP in Figure \ref{figure-example},
we assign the features $ \{(1, 0)^{\top}, (2, 0)^{\top}, (0, 1)^{\top}, (0, 2)^{\top}\}$ to the state-action pairs $\{(s_1,\mathtt{right}),(s_2,\mathtt{right}),(s_1,\mathtt{left}),(s_2,\mathtt{left})\}$,
From the dynamic transition shown in Figure \ref{figure-example} , we have
\[
P^{\pi} = \begin{pmatrix}
0 & 1 & 0 & 0\\
0 & 1 & 0 & 0\\
1 & 0 & 0 & 0\\
1 & 0 & 0 & 0
\end{pmatrix},
\Phi=
\begin{pmatrix}
1 & 0\\
2 & 0\\
0 & 1\\
0 &  2
\end{pmatrix},
\Xi=\dfrac{1}{2} I_{4\times4}.
\]
Then, according to (\ref{def:A}), we have 
\[
A=
\Phi^{\top}\Xi(I-\gamma\lambda P^{\pi})^{-1}(\gamma P^{\pi}-I)\Phi
=
 \begin{pmatrix}
\frac{6\gamma-\gamma\lambda-5}{2(1-\gamma\lambda)} & 0 \\
\frac{3\gamma}{2} & -\frac{5}{2} 
\end{pmatrix},
\] 
and the eigenvalues of $A$ are: $\frac{6\gamma-\gamma\lambda-5}{2(1-\gamma\lambda)}$ and $\frac{5}{2}$.
For any initial $\theta_{0}=\begin{pmatrix}\theta_{0,1}, \theta_{0,2}\end{pmatrix}^{\top}$, let $\mathbb{E}[\theta_{t+1}]=:(\theta_{t+1,1},\theta_{t+1,2})^{\top}$ be the expectation of iteration (\ref{gradient-ES}), 
according to (\ref{expected-para-es}), the first component of $\mathbb{E}[\theta_{t+1}|\theta_t]$ is:
\begin{flalign}
\label{example-iteration-1}
\theta_{t+1,1}=&\theta_{0,1}\prod_{i=0}^{t}\Big(1+\alpha_{i} \frac{6\gamma-\gamma\lambda-5}{2(1-\gamma\lambda)}\Big).
\end{flalign}
For any $\lambda\in(0,1)$, if $\gamma\in(\frac{5}{6-\lambda},1)$, then $\frac{6\gamma-\gamma\lambda-5}{2(1-\gamma\lambda)}$ is a positive scalar, which implies $A$ can not be a negative matrix.
Furthermore, if step size $\alpha_{t}: \sum_{i\ge0}\alpha_{t}=\infty$,
we have 
\footnote {
Eq.(\ref{THeta-0-1}) is a direct result of the following conclusion that could be found in any calculus textbook.
Let $p_{i}=1+a_i$, where $a_{i}>0$, if $\sum_{i=1}^{\infty}a_{i}=+\infty$, then 
$
\prod_{i=1}^{\infty}p_{i}=\prod_{i=1}^{\infty}(1+a_i)=+\infty.
$
}
\begin{flalign}
\label{THeta-0-1}
|\theta_{t+1,1}|=|\theta_{0,1}|\prod_{i=0}^{t}\Big(1+\alpha_{i} \frac{6\gamma-\gamma\lambda-5}{2(1-\gamma\lambda)}\Big)\rightarrow+\infty,
\end{flalign}
which implies the way (\ref{gradient-ES}) to extend $\mathtt{Expected~Sarsa}$($\lambda$) with linear function approximation via off-line estimate is unstable for off-policy learning.
\end{example}

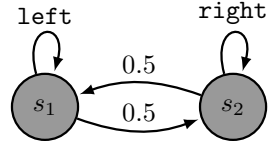
\begin{SCfigure}[4][t]
    \begin{tikzpicture}
\tikzset{node style/.style={state, 
                                    minimum width=0.1cm,
                                    line width=0.3mm,
                                    fill=gray!80!white}}

        \node[node style] at (0, 0)     (bull)     {$s_1$};
        \node[node style] at (2.5, 0)     (bear)     {$s_2$};

        \draw[every loop,
              auto=right,
              line width=0.3mm,
              >=latex,
              draw=black,
              fill=black]
            (bull)     edge[loop above] node {$\mathtt{left}$} (bull)
            (bear)     edge[loop above] node {$\mathtt{right}$} (bear)
            (bull)     edge[bend right=20, auto=left] node {$0.5$} (bear)
            (bear)     edge[bend right=20]            node {$0.5$} (bull);
    \end{tikzpicture}
    \caption{Counterexample, Two-State MDP: behavior policy $\mu(\mathtt{right} | \cdot)=0.5$
    and target policy $\pi(\mathtt{right} |\cdot)=1$.}
    \label{figure-example}
\end{SCfigure}

\section{On-Line Gradient Expected Sarsa($\lambda$)}

\label{algorithm-ges}

The above discussion of the instability for off-policy learning shows that we should abandon the off-line update (\ref{gradient-ES}).
In this section, we provide a convergent on-line algorithm: $\mathtt{Gradient ~Expected~Sarsa(\lambda)}$ ($\mathtt{GES}(\lambda)$), which is based on the popular TD fixed point method.

The TD fixed point method \cite{sutton2009fast_a,bertsekas2011temporal,dann2014policy} is widely used for policy evaluation and it focuses on finding the value function satisfies 
\begin{flalign}
\label{td-oprator}
\Phi\theta=\Pi\mathcal{B}_{\lambda}^{\pi}\Phi\theta,
\end{flalign}
where $\Pi = \Phi(\Phi^{\top}\Xi\Phi)^{-1}\Phi^{\top}\Xi$.
It has been shown that if the \emph{projected Bellman operator} $\Pi\mathcal{B}_{\lambda}^{\pi}$ has a fixed point $\theta^{\star}$, then it is unique \cite{lagoudakis2003least,bertsekas2011temporal}, and
such a fixed-point $\theta^{\star}$ also satisfies the linear system (\ref{optimal-sol}).

Instead of using the method of value iteration according to the projected Bellman operator $\Pi\mathcal{B}_{\lambda}^{\pi}$, 
we derive the algorithm on the mean square projected Bellman equation (MSPBE) \cite{sutton2009fast_a} as follows,
\begin{flalign}
\nonumber
\min_{\theta}\text{MSPBE}(\theta,\lambda)&=:\min_{\theta}\frac{1}{2}\|\Phi\theta-\Pi\mathcal{B}^{\pi}_{\lambda}(\Phi\theta)\|^{2}_{\Xi}\\
\label{Eq:mspbe}
&=\min_{\theta}\frac{1}{2}\|A\theta+b\|^{2}_{M^{-1}},
\end{flalign}
where $\|x\|_{\Xi}=x^{\top}\Xi x$ is a weighted norm,
and $M=\Phi^{\top}\Xi\Phi$.
We provide the derivation of (\ref{Eq:mspbe}) in Appendix C.1.

Since the computational complexity of the invertible matrix $M^{-1}$ is very large, it is too expensive to
use stochastic gradient method to solve the problem (\ref{Eq:mspbe}) directly. 
Let
\begin{flalign}
\label{Psi}
g(\omega)&=\frac{1}{2}\|\omega\|^{2}_{M}-b^{\top}\omega\\
\nonumber
\Psi(\theta,\omega)&=(A\theta+b)^{\top}\omega-\frac{1}{2}\|\omega\|^{2}_{M}=\theta^{\top}A\omega-g(\omega).
\end{flalign}
According to \citet{liu2015finite}, the original problem (\ref{Eq:mspbe}) is equivalent to the
convex-concave saddle-point problem
\begin{flalign}
\label{saddle-point-problem}
\min_{\theta}\max_{\omega}\{\Psi(\theta,\omega)\}.
\end{flalign}
\begin{proposition}
\label{propo-1}
If $(\theta^{\star},\omega^{\star})$ is the solution of problem (\ref{saddle-point-problem}), then $\theta^{\star}$ is the solution of original problem (\ref{Eq:mspbe}), i.e., \[\theta^{\star}=\arg\min_{\theta}\emph{MSPBE}(\theta,\lambda).\]
\end{proposition}
We provide the proof of Proposition \ref{propo-1} in Appendix C.2.
Proposition \ref{propo-1} illustrates that the solution of (\ref{Eq:mspbe}) is contained in the problem (\ref{saddle-point-problem}).
Gradient update is a natural way to solve problem (\ref{saddle-point-problem}) (ascending in $\omega$ and descending in $\theta$):
\begin{flalign}
\label{E-al-0}
    \omega_{t+1}&=\omega_t+\beta_t(A\theta_t+b-M\omega_t),\\
    \label{E-al}
    \theta_{t+1}&=\theta_t-\alpha_t A^{\top}\omega_t,
\end{flalign}
where $\alpha_t,\beta_t$ is step-size, $t\ge0$.

\begin{algorithm}[t]
    \caption{Gradient Expected Sarsa$(\lambda)$ ($\mathtt{GES}(\lambda)$)}
    \label{alg:algorithm1}
    \begin{algorithmic}[1]
        \STATE { \textbf{Initialization}: $\omega_{0}=0$, ${\theta}_{0}=0$, $\alpha_0>0,\beta_0>0$, $T\in\mathbb{N}$.} \\
        \STATE ${e}_{-1}={0}$
        \FOR{$t=0$ {\bfseries to} $T$}
        \STATE Observe $\{S_{t},A_{t},R_{t+1},S_{t+1},A_{t+1}\}\sim\mu$
         \STATE$\rho_{t}=\frac{\pi(A_{t}|S_{t})}{\mu(A_{t}|S_{t})}$
        \STATE $e_{t}=\lambda\gamma \rho_{t}e_{t-1}+\phi_{t}$
        \STATE $\delta_{t}=R_{t+1}+\gamma{\theta}_{t}^{\top}\mathbb{E}_{\pi}\phi(S_{t+1},\cdot)-{\theta}_{t}^{\top}\phi_t$
        \STATE $\omega_{t+1}=\omega_{t}+\beta_t(e_{t}\delta_t-\phi_t\phi_{t}^{\top}\omega_t)$
        \STATE $\theta_{t+1}=\theta_{t}-\alpha_t (\gamma\mathbb{E}_{\pi}[\phi(S_{t+1,\cdot})]-\phi_{t})e_t^{\top}\omega_{t}$
        \ENDFOR
        \STATE { \textbf{Output}:$\{\theta_t,\omega_t\}_{t=1}^{T}$}
    \end{algorithmic}
\end{algorithm}

However, since $A, b$, and $ M$ are versions of expectations, we can not get the transition probability in practice. A practical way is to find the unbiased estimators of them.
Let $e_{0}=0, \rho_{t}=\frac{\pi(A_{t}|S_{t})}{\mu(A_{t}|S_{t})}, e_{t}=\lambda\gamma \rho_{t}e_{t-1}+\phi_{t}, \hat{b}_{t}=R_{t+1}e_{t}, \hat{A}_{t}=e_{t}(\gamma\mathbb{E}_{\pi}[\phi(S_{t+1,\cdot})]-\phi_{t})^{\top},\hat{M}_{t}=\phi_{t}\phi^{\top}_{t}$. According to the Theorem 9 of \cite{maei2011gradient}, we have  
\begin{flalign}
\label{unbiased-A-b-M}
\mathbb{E}_{\mu}[\hat{A}_{t}]=A,\mathbb{E}_{\mu}[\hat{b}_{t}]=b,\mathbb{E}_{\mu}[\hat{M}_{t}]=M.
\end{flalign}
Replacing the expectations in (\ref{E-al-0}) and (\ref{E-al}) by corresponding unbiased estimates, we define the stochastic on-line implementation of (\ref{E-al-0}) and (\ref{E-al}) as follows,
\begin{flalign}
\label{stochastic-im-1}
\omega_{t+1}&=\omega_{t}+\beta_t(\hat A_{t}{\theta}_{t}+\hat b_{t}-\hat M_{t}\omega_{t}) ,\\
\label{stochastic-im}
\theta_{t+1}&=\theta_{t}-\alpha_t \hat A_{t}^{\top}\omega_{t}.
\end{flalign}
We provide more details in Algorithm \ref{alg:algorithm1}.

\section{Finite-Time Performance Analysis}

In this section, we mainly focus on the finite-time performance of $\mathtt{GES}(\lambda)$.
Theorem \ref{linear-convergence} shows the proposed $\mathtt{GES}(\lambda)$ converges at a linear rate under a concrete step-size, which is comparable to extensive existing state-of-the-art GTD algorithms. We have provided an adequate comparison in Table \ref{table-1} and a comprehensive discussion in the section of related works.
Furthermore, to investigate how the step-size influences finite-time performance, we establish Theorem \ref{step-size-per}.
We develop a technique of Lyapunov function to prove Theorem \ref{step-size-per}, before we present the main result, we provide the motivation and some necessary details of Lyapunov function.

Throughout this paper, we make two additional standard assumptions, which are widely used in reinforcement learning \cite{wang2017finite,bhandari2018finite,xu2019two}.
\begin{assumption}[Boundedness of Feature Map, Reward]
\label{ass:Boundedness-of-Parameter}
    The features $\{\phi_{t}\}_{t\ge0}$ is uniformly bounded by $\phi_{\max}$. 
    The reward function is uniformly bounded by $R_{\max}$.
    The importance sampling $\rho_{t}$ is uniformly bounded by $\rho_{\max}$.
\end{assumption}

\begin{assumption}[Solvability of Problem]
\label{ass:Solvability-of-Problem}
The matrix $A$ is non-singular and \emph{rank}$(\Phi)=p$.
\end{assumption}

As claimed by Xu et al., \shortcite{xu2019two},
Assumption \ref{ass:Boundedness-of-Parameter} can be ensured by normalizing the feature maps $\{\phi_{t}\}_{t\ge1}$ and when $\mu(\cdot|s)$ is non-degenerate for all $s\in\mathcal{S}$.  
Besides,
Assumption \ref{ass:Boundedness-of-Parameter} implies the boundedness of the estimators $\hat{A}_{t}$, $\hat{M}_{t}$ and $\hat{b}_{t}$.
For the limitation of space,
we provide more details and discussions in Remark \ref{app-reark-01} (see Appendix D.1).

Assumption \ref{ass:Solvability-of-Problem} requires the non-singularity of the matrix $A$, which implies
the optimal parameter $\theta^{\star} =-A^{-1}b$ is well defined. 
The feature matrix $\Phi$ has linearly independent columns implies the matrix $M$ is non-singular.

\subsection{Linear Convergence Rate}
We consider the first-order optimality condition of the problem (\ref{saddle-point-problem}), i.e., the optimal solution $(\theta^{\star},\omega^{\star})$ satisfies
\begin{flalign}
\label{first-order optimality condition-app}
 \begin{cases}
     \nabla_{\theta} \Psi(\theta^{\star},\omega^{\star}) &=A^{\top}\omega^{\star}=0,\\
     \nabla_{\omega} \Psi(\theta^{\star},\omega^{\star}) & =-\nabla g(\omega^{\star})  +A\theta^{\star}=0.
 \end{cases}
\end{flalign}
According to the Fact \ref{fact1} and the condition (\ref{first-order optimality condition-app}), we have
\[\omega^{\star}=(\nabla g)^{-1}(A\theta^{\star})=\nabla g^{\star}(A\theta^{\star}),\]
which implies $\omega^{\star}$ can be represented by $\theta^{\star}$, thus, we mainly focus on the performance of $\{\theta_t\}_{t\ge1}$.

\begin{theorem}
\label{linear-convergence}
$\{(\theta_{t},\omega_{t})\}_{t\ge0}$ is generated by Algorithm\ref{alg:algorithm1}.
Let
$\Delta_{\theta_{t}}=\|\theta_{t}-\theta^{\star}\|^{2}_{2}$, $\Delta_{\omega_{t}}=\|\omega_{t}-\nabla g^{\star}(A\theta_{t})\|^{2}_{2}$, $\nu=\dfrac{2\kappa^{2}(A)\kappa(M)\lambda_{\max}(A)}{\lambda_{\min}(M)}$,and $ D_{t}=\nu \Delta_{\theta_{t}}+ \Delta_{\omega_{t}}$.
If we choose step-size
$
\alpha=\frac{\lambda_{\min}(M)}{\big(\lambda_{\max}(M)+\lambda_{\min}(M)\big)\big(\frac{\lambda^{2}_{\max}(A)}{\lambda_{\min}(M)}+\nu\lambda_{\max}(A)\big)},
\beta=\frac{2}{\lambda_{\max}(M)+\lambda_{\min}(M)}
$, under Assumption \ref{ass:ergodicity}-\ref{ass:Solvability-of-Problem}, we have
\begin{flalign}
\label{P_t-1}
\mathbb{E}[D_{t+1}]\leq \Big(1-\frac{1}{12\kappa^{3}(M)\kappa^{4}(A)}\Big) \mathbb{E}[D_{t}].
\end{flalign}
Furthermore, we have
\begin{flalign}
\label{theTA_t-gap}
\mathbb{E}[\|\theta_{t}-\theta^{\star}\|^{2}_{2}]\leq \frac{1}{\nu}\Big(1-\frac{1}{12\kappa^{3}(M)\kappa^{4}(A)}\Big)^{t}\mathbb{E}[D_{0}].
\end{flalign}
\end{theorem}
\begin{remark}
\label{remark-sec-com}
We provide its proof in Appendix C.3.
Recall the fact $D_{t}=\nu \Delta_{\theta_{t}}+ \Delta_{\omega_{t}}\ge\nu \Delta_{\theta_{t}}$, which implies the inequality
\[
\mathbb{E}[\|\theta_{t}-\theta^{\star}\|^{2}_{2}]\leq\frac{\mathbb{E}[D_{t}]}{\nu}\overset{(\ref{P_t-1})}\leq \dfrac{1}{\nu}\Big(1-\frac{1}{12}\frac{1}{\kappa^{3}(M)\kappa^{4}(A)}\Big)^{t}\mathbb{E}[D_{0}].
\]
The term $(1-\frac{1}{12}\frac{1}{\kappa^{3}(M)\kappa^{4}(A)})\in(0,1)$ implies $\mathtt{GES}(\lambda)$ produces the sequence $\{\theta_t\}_{t\ge0}$ converges to the optimal solution at a linear convergence rate. 
After some simple algebra, with a computational cost of
\[
\mathcal{O}
\bigg(
\max\Big\{1,\frac{\lambda_{\max}(A)}{\lambda_{\max}(M)\nu}\Big\}
\Big(1-\frac{1}{12\kappa^{3}(M)\kappa^{4}(A)}\Big)\log\frac{1}{\delta}
\bigg),
\]
the output of Algorithm \ref{alg:algorithm1} closes to $(\theta^{\star},\omega^{\star})$ as follows,
\[
\mathbb{E}[\|\theta_t-\theta^{\star}\|_2^{2}]\leq\delta^2,~~~\mathbb{E}[\|\omega_t-\omega^{\star}\|_2^{2}]\leq\delta^2.
\]
\end{remark}
\begin{remark}
Theorem \ref{linear-convergence} provides a concrete step-size that depends on the some unknown parameters.
As suggested by \citet{du2017stochastic}, \citet{touati2018convergent}, and \citet{voloshin2019empirical}, we can use Monte Carlo method to estimate the unknown parameters.
\end{remark}

\subsection{A Further Analysis via Lyapunov Function.}

In this section, we propose Theorem \ref{step-size-per} that illustrates a relationship between the performance of $\mathtt{GES}(\lambda)$ and step-size.

The proof of Theorem \ref{step-size-per} involves a novel Lyapunov function technique, we start by presenting the motivation behind such Lyapunov function.
Let
\[
H=-{A}^{\top}M^{-1}{A},L=2A.
\]
Let $Q_1$ and $Q_2$ be the solutions of the following equations
\begin{flalign}
\label{PQ-EQ-app-sec}
 \begin{cases}
    -H^{\top}Q_1-Q_1H&=I\\
    M^{\top}Q_2+Q_2M&=I.
 \end{cases}
\end{flalign}
Since both $M$ and $-H$ are \emph{Hurwitz} matrix, the solution of the linear system (\ref{PQ-EQ-app-sec}) alway exists \cite{lakshminarayanan2018linear,srikant2019-colt-finite-time}.
Let $Q$ be a matrix as follows,
\[
Q=\dfrac{1}{p_1+p_2}
\begin{pmatrix}
p_1
&0\\
0&p_2
\end{pmatrix},
\]
where $p_1=\|Q_1A^{\top}\|_{op}Q_1$, $p_2=\|Q_2M^{-1}AL\|_{op}Q_2$.
Finally, we define $\varrho_{t}$ and $z_t$ as follows,
\begin{flalign}
\label{online-revisit-sec}
\varrho_{t}&=\omega_t-M^{-1}A\theta_{t}, ~~~
z_t=(
\theta_{t}-\theta^{\star},
\varrho_{t}-\varrho^{\star}
)^{\top},
\end{flalign}
where $\varrho^{\star}=\omega^{\star}-M^{-1}A\theta^{\star}$.
Lyapunov function $L(z_t)$ is:
\begin{flalign}
\label{lyapunov-function-sec}
L(z_t)=z_t ^{\top}Q z_t.
\end{flalign}

\textbf{Motivation of Lyapunov Function}
We consider the expected difference of iteration (\ref{stochastic-im-1})-(\ref{stochastic-im}) as follows
\begin{flalign}
\label{app-sec-01}
\frac{1}{\alpha}\mathbb{E}[\theta_{t+1}-\theta_t|Y_{t-\tau}]=&\mathbb{E}[-\hat{A}_t \omega_t |Y_{t-\tau}]\\
\label{app-sec-02}
\frac{\alpha}{\beta}\frac{\mathbb{E}[\omega_{t+1}-\omega_t|Y_{t-\tau}]}{\alpha}=&\mathbb{E}[\hat A_{t}{\theta}_{t}+\hat b_{t}-\hat M_{t}\omega_{t} |Y_{t-\tau}],
\end{flalign}
where $Y_{t-\tau}=\{\theta_{t-\tau},\omega_{t-\tau},X_{t-\tau}\}$, and the sequence
$
X_{t}=:\{S_0,A_0,S_1,A_1,\cdots,S_t,A_t\}
$
denotes a Markov chain according to Algorithm \ref{alg:algorithm1}.
The expectation is conditioned sufficiently in the past information of the underlying Markov chain.
Approximating the left parts of (\ref{app-sec-01})-(\ref{app-sec-02}) by derivatives, then we have the ordinary differential equation (ODE):
\begin{flalign}
\label{o-d-e-sec}
 \begin{cases}
\dot\theta(t)=-A\omega(t),\\
\frac{\alpha}{\beta}\dot\omega(t)=A{\theta}({t})+ b-M\omega({t}),
 \end{cases}
\end{flalign}
where $\theta(t),\omega(t)\in\mathbb{R}^{p}$ of (\ref{o-d-e-sec}) are the functions that are defined on the continuous time $(0,\infty)$.
The update rule of (\ref{stochastic-im-1})-(\ref{stochastic-im}) can be thought of as a discretization 
of the ODE (\ref{o-d-e-sec}) that is known as \emph{singular perturbation} ODE (chapter 11 of \cite{khalil2002nonlinear}), 
our goal is to provide a non-asymptotic analysis of $\mathtt{GES}(\lambda)$ according to the asymptotically stable equilibrium of ODE  (\ref{o-d-e-sec}).
According to Khalil and Grizzle \shortcite{khalil2002nonlinear},
the following Lyapunov function $L(t)$ is widely used as a stability criteria for the ODE (\ref{o-d-e-sec}),
\begin{flalign}
\nonumber
L(a,t)=& a\big(\omega(t)-M^{-1}A\theta(t)\big)^{\top}Q_2\big(\omega(t)-M^{-1}A\theta(t)\big)\\
\label{continu-time-lyapu}
&+(1-a)\theta^{\top}(t)Q_{1}\theta(t),
\end{flalign}
where $a\in(0,1)$.
Our $L(z_t)$ (\ref{lyapunov-function-sec}) can be seen as a discretization of $L(t)$ (\ref{continu-time-lyapu}) after a proper choice of $a$, which inspires us to conduct the Lyapunov function $L(z_t)$.
\begin{lemma}
\label{lyapunov-function-improve-one-step}
Under Assumption \ref{ass:ergodicity}-\ref{ass:Solvability-of-Problem}, there exists a positive scalar $\tau$ such that: $t\ge\tau$,
\begin{flalign}
\nonumber
\mathbb{E}[L(z_{t+1})]-\mathbb{E}[L(z_{t})]
&\leq-\alpha\Big(\frac{1}{2}\varkappa_1
-\frac{\alpha}{\beta}\varkappa_2\Big)
\mathbb{E}[L(z_t)]\\
\label{lemma-lypi}
&~~~~~~+2\alpha^2\zeta^{2}\lambda_{\max}(Q)\widetilde{c_{b}}^2,
\end{flalign}
where the constants $\varkappa_1,\varkappa_2,\zeta=:C_1+\dfrac{\beta}{\alpha}C_2,\widetilde{c_{b}}$ are defined in Appendix D.
\end{lemma}
Finally, we know \[\mathbb{E}[\|z_t\|_2^{2}]\leq(\lambda_{\min}(Q))^{-1}\mathbb{E}[L(z_{t})],\] applying the result of (\ref{lemma-lypi}) recurrently, we have Theorem \ref{step-size-per}.
\begin{theorem}
\label{step-size-per}
Let
$\eta_1=4\zeta^2\tau^2(\|z_0\|_2+\widetilde{c_b})^2+\|z_0\|_2^2)$, $\eta_2=\frac{2\kappa(Q)\zeta^{2}\lambda_{\max}(Q)\widetilde{c_{b}}^2}{\frac{1}{2}\varkappa_1-\frac{\alpha}{\beta}\varkappa_2}$.
Under Assumption \ref{ass:ergodicity}-\ref{ass:Solvability-of-Problem}, there exists a positive scalar $\tau$ such that: $t\ge\tau$,
\begin{flalign}
\nonumber
\mathbb{E}[\|z_t\|_2^{2}]&\leq\alpha^2\eta_1\Big(1-\frac{\alpha}{\lambda_{\max}(Q)}\Big(\frac{1}{2}\varkappa_1-\frac{\alpha}{\beta}\varkappa_2\Big)\Big)
^{t-\tau}+\alpha\eta_2.
\end{flalign}
\end{theorem}
\begin{remark}
Recall $z_t=(\theta_{t}-\theta^{\star},\varrho_{t}-\varrho^{\star})^{\top}$, then $\mathbb{E}[\|\theta_{t}-\theta^{\star}\|^2_2]\leq\mathbb{E}[\|z_t\|_2^{2}],$
which implies the expected mean square error of $\theta_{t}-\theta^{\star}$ is also upper-bounded by the result of Theorem \ref{step-size-per}.
Furthermore, after a total computational cost of 
\[\tau+\mathcal{O}\big(\frac{1}{\delta}\log\frac{1}{\delta}\big),\]
the Algorithm \ref{alg:algorithm1} outputs $\theta_t$ closes to the optimal solution $\theta^{\star}$ as follows,
\[\mathbb{E}[\|\theta_{t}-\theta^{\star}\|_2^2]\leq\mathcal{O}(\tau\delta).\]
\end{remark}
\begin{remark}
Theorem \ref{step-size-per} shows that the upper-bounded error consists of two different parts:
the first error bound depends on both step-size and the size of samples, and this error decays geometrically as the number of iteration increases;
while the second part is only determined by the step-sizes and it is independent of the number of iterations.
\end{remark}

\begin{table*}[t]
\centering
   \begin{tabular}{|p{4.5cm}|p{3.3cm}|c |c |c|}

         \hline
        \makecell{Algorithm}  &\makecell{Step-size} & \makecell{Convergence Rate}&\makecell{ TD Fixed Point}\\
        \hline
        \makecell{$\mathtt{TD}(0)$\\\cite{Nathaniel2015ontd}}&\makecell{$\alpha_{t}=\mathcal{O}({t^{-\eta}})$\\ $\eta\in(0,1)$}&$\mathcal{O}\big(\dfrac{1}{\sqrt{T}}\big)$&$\Phi^{\top}\Xi(\gamma P^{\mu}-I)\Phi\theta^{\star}=-b$\\
          \hline
        \makecell{$\mathtt{TD}(0)$\\\cite{gal2018finite}}&\makecell{$\sum_{t=1}^{\infty}\alpha_t=\infty$}&\makecell{$\mathcal{O}\big(\dfrac{1}{T^{\eta}}\big)$\\$\eta\in(0,1)$}&$\Phi^{\top}\Xi(\gamma P^{\mu}-I)\Phi\theta^{\star}=-b$\\
          \hline
          \makecell{$\mathtt{TD}(0)$\\\cite{lakshminarayanan2018linear}}    &\makecell{Constant }& \makecell{$\mathcal{O}\big(\dfrac{1}{T}\big)$ } &$\Phi^{\top}\Xi(\gamma P^{\mu}-I)\Phi\theta^{\star}=-b$\\
            \hline
         \makecell{$\mathtt{GTD}(0)$\\\cite{gal2018finitesample}}&
        \makecell{$\sum_{t=1}^{\infty}\alpha_t=\infty$,$\frac{\beta_{t}}{\alpha_{t}}\rightarrow0$}&\makecell{$\mathcal{O}\big(\big(\dfrac{1}{T}\big)^{\frac{1-\kappa}{3}}\big)$\\$\kappa\in(0,1)$}&$\Phi^{\top}\Xi(\gamma P^{\pi}-I)\Phi\theta^{\star}=-b$\\
        \hline
      \makecell{  $\mathtt{GTD}(0)$/$\mathtt{GTD2}$/$\mathtt{TDC}$\\  \cite{dalal2019tale} }    &\makecell{$\alpha_{t}=\frac{1}{t^{\eta_1}}$,$\beta_{t}=\frac{1}{t^{\eta_2}}$ \\$0<\eta_2<\eta_1<1$}& $\mathcal{O}\big(\dfrac{1}{T^{\eta_1}}\big)$     
        &$\Phi^{\top}\Xi(\gamma P^{\pi}-I)\Phi\theta^{\star}=-b$\\
          \hline
        \makecell{  $\mathtt{GTB}(\lambda)$\\\cite{touati2018convergent} }    &   \makecell{$\alpha_{t},\beta_{t}=\mathcal{O}(\frac{1}{t})$}   & $\mathcal{O}\Big(\dfrac{1}{T}\Big)$  &$\Phi^{\top}\Xi(I-\gamma\lambda P^{\mu})^{-1}(\gamma P^{\pi}-I)\Phi\theta^{\star}=-b$\\
          \hline
        \makecell{ $\mathtt{SARSA}$ \\ \cite{zou2019finite} }   & \makecell{$\alpha_{t}=\mathcal{O}(\frac{1}{t})$} & $\mathcal{O}\Big(\dfrac{\log^{3} (T)}{T}\Big)$&$\Phi^{\top}\Xi(\gamma P^{\mu}-I)\Phi\theta^{\star}=-b$     \\
          \hline
          \makecell{  $\mathtt{TDC}$\\ \cite{xu2019two} }&  \makecell{
       $
       \max\{\alpha_t\log(\frac{1}{\alpha_t}),\alpha_{t}\}$\\$\leq\min\{\frac{\|\theta_{0}-\theta^{\star}\|_{2}}{2^{t-1}},C\}
      $
      }
         &   Linear &$\Phi^{\top}\Xi(\gamma P^{\pi}-I)\Phi\theta^{\star}=-b$ \\
          \hline
          \makecell{  $\mathtt{TDC}$\\ \cite{kaledin2020finite} }&  \makecell{    
 $\alpha_t=\frac{1}{t^{v}}$, $\beta_{t}=\frac{1}{t}$\\ $v\in(0,1)$   
    }
         &   $\mathcal{O}(\dfrac{1}{T} )$&$\Phi^{\top}\Xi(\gamma P^{\pi}-I)\Phi\theta^{\star}=-b$ \\
         \hline
          \rowcolor{blue!5}
        \makecell{$\mathtt{GES(\lambda)}$\\Theorem \ref{linear-convergence} }&    \makecell{Constant} & $\mathcal{O}\big((1-\frac{C}{12})^{T}\big)$ &$\Phi^{\top}\Xi(I-\gamma\lambda P^{\pi})^{-1}(\gamma P^{\pi}-I)\Phi\theta^{\star}=-b$  \\
       \hline
    \end{tabular}
     \caption{
         Comparison of GTD family algorithms over performance measurement $\mathbb{E}\|\theta_{T}-\theta^{\star}\|_{2}^{2}$.
         }    
          \label{table-1}   
        \end{table*}

\section{Related Works}

In this section, we review existing finite-time performance of GTD algorithms over $\|\theta_{T}-\theta^{\star}\|_{2}^{2}$.

Although the asymptotic analysis of GTD family has been established in \cite{sutton2009fast_a,sutton2009convergent_b,maei2011gradient}, which holds only in the limit as the number of iterations increases to infinity, and we can not get the information of convergence rate from asymptotic results.
This is the main reason why we focus on the finite-time performance over $\|\theta_{T}-\theta^{\star}\|_{2}^{2}$.
It is noteworthy that \citet{liu2015finite} firstly introduce \emph{primal-dual gap error} to measure the convergence of GTD algorithm, we provide the discussion of finite-time primal-dual gap error analysis in Appendix E.

\citet{Nathaniel2015ontd} proves that $\mathtt{TD(0)}$ \cite{sutton1988learning} converges at $\mathcal{O}(\frac{1}{\sqrt{T}})$ with the step-size 
$\alpha_{t}=\mathcal{O}(\frac{1}{t^{\eta}})$, $\eta\in(0,1)$.
Later, \citet{gal2018finite} further explore the property of $\mathtt{TD(0)}$, they prove the convergence rate of $\mathtt{TD(0)}$ achieves
$\mathcal{O}(e^{-\frac{\lambda}{2}T^{1-\eta}}+\frac{1}{T^{\eta}})$, but it never reaches $\mathcal{O}(\frac{1}{T})$, where $\eta\in(0,1)$, $\lambda$ is the minimum eigenvalue of the matrix $A^{\top}+A$. 
Lakshminarayanan et al., \shortcite{lakshminarayanan2018linear} show $\mathtt{TD(0)}$ converges at $\mathcal{O}(\frac{1}{T})$ with a more relaxed step-size than the works of \cite{Nathaniel2015ontd,gal2018finite}, it only requires a constant step-size.
Recently, Dalal et al. \shortcite{gal2018finitesample} proves $\mathtt{GTD(0)}$ family algorithm \cite{sutton2009convergent_b} converges at $\mathcal{O}((\frac{1}{T})^{\frac{1-\kappa}{3}})$, but nerve reach $\mathcal{O}(\frac{1}{T})$, where $\kappa\in(0,1)$.
A very similar convergence rate appears in \cite{dalal2019tale}, which considers $\mathtt{TDC}$ and $\mathtt{GTD2}$.
\citet{touati2018convergent} propose $\mathtt{GTB}(\lambda)$/$\mathtt{GRetrace}(\lambda)$,
they prove the convergence rate of $\mathtt{GTB}(\lambda)$/$\mathtt{GRetrace}(\lambda)$ reahches $\mathcal{O}(\frac{1}{T})$. 
\citet{zou2019finite} show $\mathtt{SARSA}$ with linear function approximation converges at the rate of $\mathcal{O}\big(\frac{\log^{3} (T)}{T}\big)$.
Recently, \citet{kaledin2020finite} further develop two timescale stochastic approximation with Markovian noise, 
and they show that $\mathtt{TDC}$ converges at a rate of $\mathcal{O}(\frac{1}{T})$ if $\alpha_t=\frac{1}{t^{v}}$, $\beta_{t}=\frac{1}{t}$, $v\in(0,1)$.

Theorem \ref{linear-convergence} illustrates our $\mathtt{GES}(\lambda)$ achieves a linear convergence rate, 
thus $\mathtt{GES}(\lambda)$ converges faster than all above gradient TD learning algorithms theoretically.
Although Xu et al., \shortcite{xu2019two} prove $\mathtt{TDC}$ also converges at a linear convergence rate, 
they require a fussy blockwise diminishing step-size condition:
$ \max\{\alpha_t\log(\frac{1}{\alpha_t}),\alpha_{t}\}\leq\{\min\{\frac{\|\theta_{0}-\theta^{\star}\|_{2}}{2^{t-1}},C\}$,
$\alpha_{t}=\mathcal{O}({\frac{1}{({t+1})^{\eta_1}}})$, $\beta_{t}=\mathcal{O}({\frac{1}{({t+1})^{\eta_2}}})$, $0<\eta_2<\eta_1<1$, where $C$ is a constant.
Apparently, our Theorem \ref{linear-convergence} requires a simpler condition of step-size than Xu et al., \shortcite{xu2019two}.
It is noteworthy that our Theorem \ref{linear-convergence} does not require an additional projection step (that is unnecessary in practice) that appears in \cite{xu2019two}.

Significantly, the finite-time performances of \cite{gal2018finite,lakshminarayanan2018linear} requires an additional assumption that all the samples required to update the function parameters are i.i.d. 
In this paper, we remove this condition and achieve a better result than theirs.

\section{Experiments}

In this section, we test the capacity of $\mathtt{GES}(\lambda)$ for off-policy evaluation in three typical domains: MountainCar, Baird Star \cite{baird1995residual}, Two-State MDP \cite{touati2018convergent}. We compare $\mathtt{GES}(\lambda)$ with three state-of-art algorithms:
$\mathtt{GQ}(\lambda)$ \cite{maei2010GQ}, 
$\mathtt{ABQ}(\zeta)$ \cite{Mahmood2017_b_multi},
$\mathtt{GTB}(\lambda)$ \cite{touati2018convergent} over two typical measurements: MSPBE and mean square error (MSE).
We choose those three algorithms as baselines since they are all learning via expected TD-error $\delta_{t}^{\text{ES}}$, which is same as $\mathtt{GES}(\lambda)$. 

\textbf{Feature Map and Parameters}
Recall the states and actions of MountainCar: $\mathcal{S}=\{(\mathtt{Velocity},\mathtt{Position})\}=[-0.07,0.07]\times[-1.2,0.6]$, $\mathcal{A}=\{\mathtt{left},\mathtt{neutral},\mathtt{right}\}$.
In this example, if $\mathtt{Velocity}>0$, we use behavior policy $\mu=(\frac{1}{100},\frac{1}{100},\frac{98}{100})$, $\pi=(\frac{1}{10},\frac{1}{10},\frac{8}{10})$;
else $\mu=(\frac{98}{100},\frac{1}{100},\frac{1}{100})$, $\pi=(\frac{8}{10},\frac{1}{10},\frac{1}{10})$.
Since the state space is continuous, we use an open \emph{tile coding}\footnote{\url{ http://incompleteideas.net/rlai.cs.ualberta.ca/RLAI/RLtoolkit/ tilecoding.html}} software to extract feature of states. 
We set the number of tilings to be 4, and there are no white noise features. 
The performance is an average 5 runs, and each run contains 5000 episodes. 
As suggested by Sutton and Barto \shortcite{sutton2018reinforcement},
we set all the initial parameters to be 0, which is optimistic about causing extensive exploration.

The Baird Star is an episodic seven states  MDP with two actions: $\mathtt{dashed}$ action and $\mathtt{solid}$ action.
In this example, we set the behavior policy $\mu(\cdot|\mathtt{dashed})=\frac{6}{7}$, $\mu(\cdot|\mathtt{solid})=\frac{1}{7}$ and 
target policy $\pi(\cdot|\mathtt{solid})=1$. We choose the feature map matrix as follows
$\Phi=
\begin{pmatrix}
&2I_{7\times 7}&\mathbbm{1}_{7\times 1}&\bm{0}_{7\times 8}\\
&\bm{0}_{7\times 8}&2I_{7\times 7}&\mathbbm{1}_{7\times 1}
\end{pmatrix}
$, where $\bm{0}$ denotes a matrix whose elements are all $0$, and $\mathbbm{1}_{7\times 1}$ denotes a vector whose elements are all $1$.
The dynamics of Two-State MDP is presented in Example \ref{two-state-example}.
We set $\lambda=0.99$, $\gamma=0.99$ in all the experiments.  
The MSPBE/MSE distribution is computed over the combination of step-size,
$(\alpha_{t},\frac{\beta_{t}}{\alpha_{t}})\in[0.1\times 2^{j}|j = -10,-9,\cdots,-1, 0]^{2}$.

\begin{figure}[t!]
    \centering
    \subfigure
    {\includegraphics[width=4cm,height=3cm]{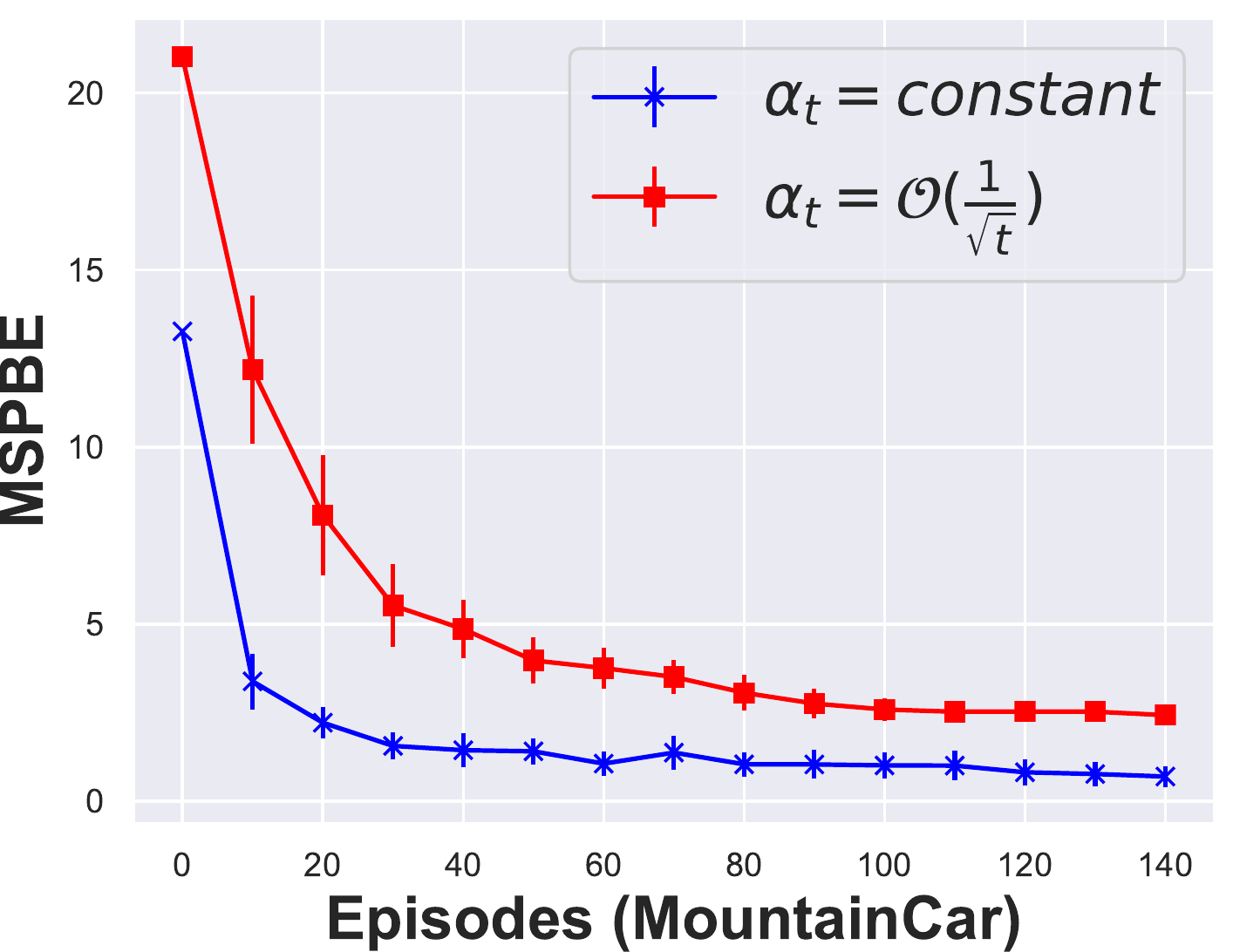}}
    \subfigure
    {\includegraphics[width=4cm,height=3cm]{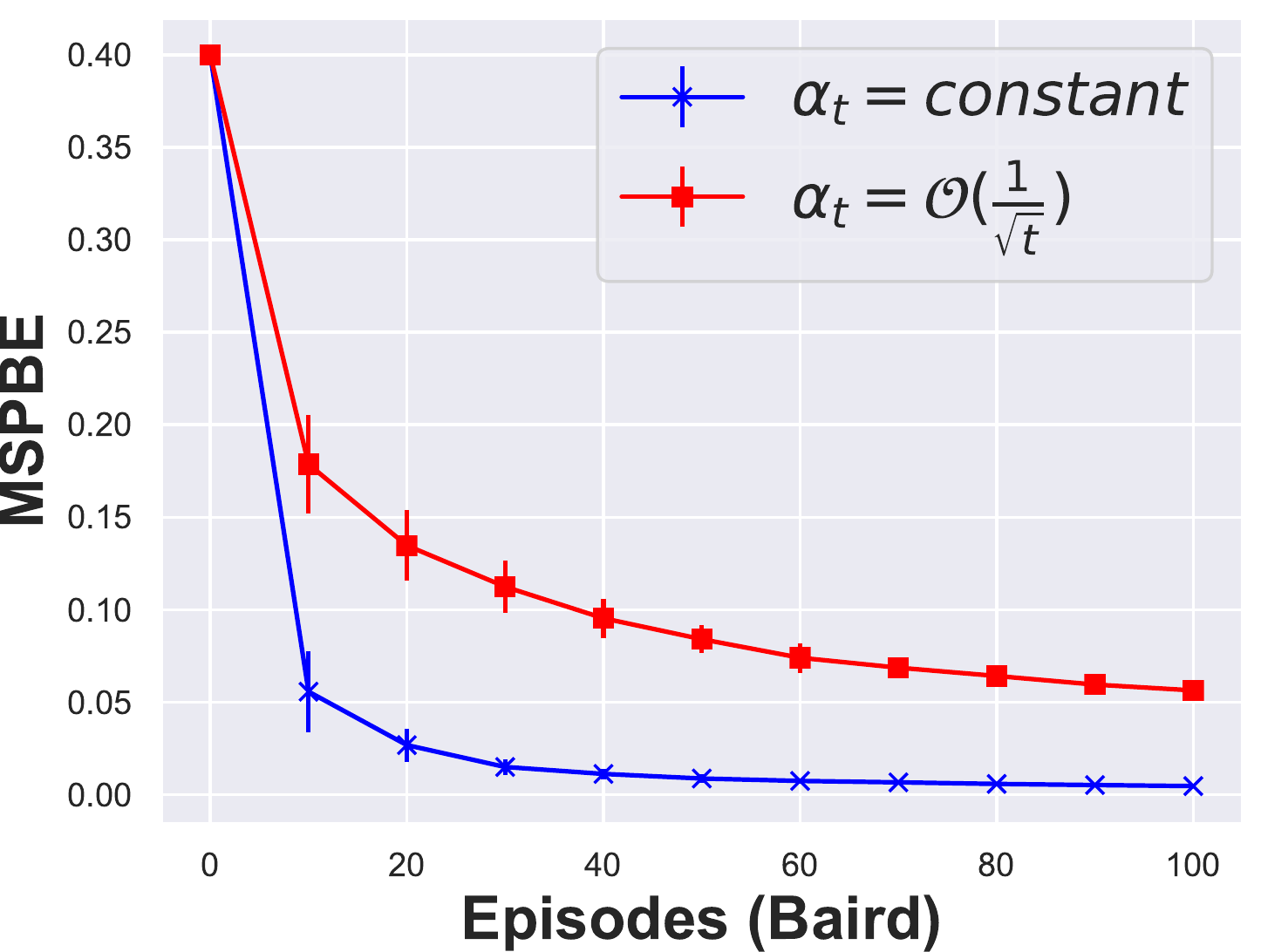}}
    \caption
    {
        Comparison between a constant step-size and $\frac{1}{\sqrt{t}}$.
    }
    \label{Step-size}
\end{figure}

\begin{figure}[t!]
    \centering
    \subfigure
    {\includegraphics[width=4cm,height=3cm]{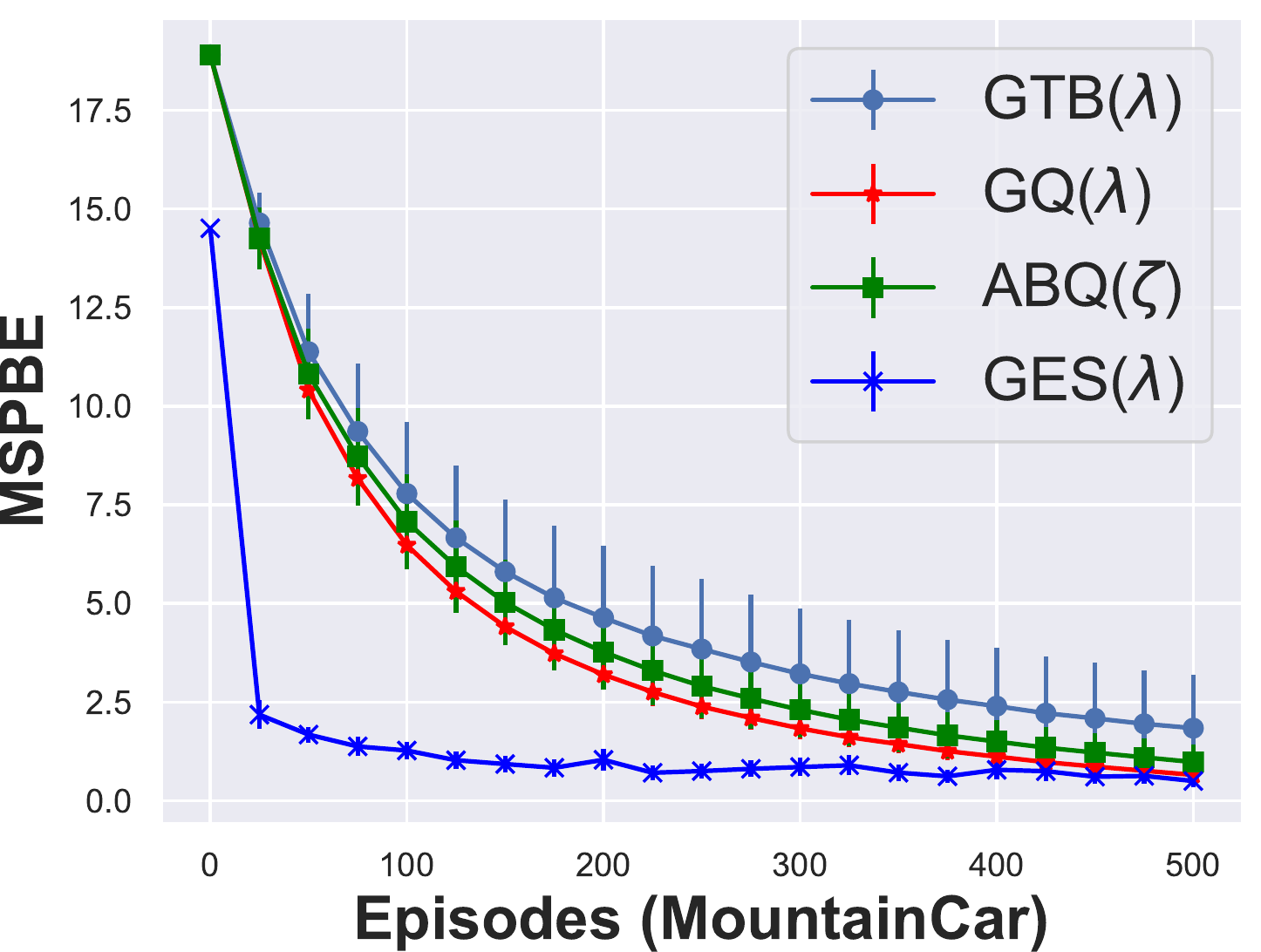}}
    \subfigure
    {\includegraphics[width=4cm,height=3cm]{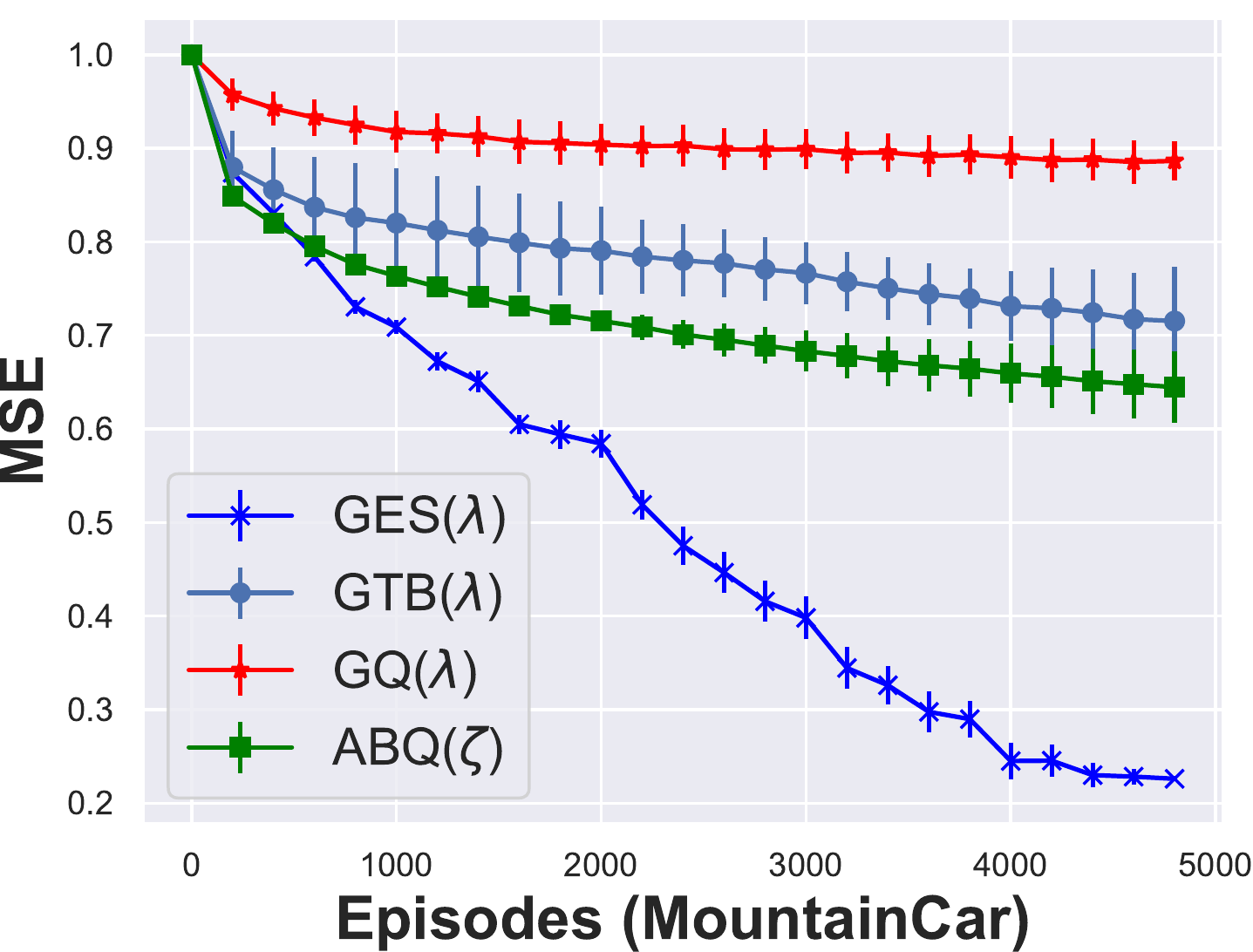}}
    \caption
    {
        MSPBE and MSE comparison on MountainCar.
    }
    \label{MC-ex-result}
\end{figure}

\textbf{Effect of Step-size}
Figure \ref{Step-size} shows the comparison of the empirical MSPBE performance between a constant step-size and the decay step-size $\mathcal{O}(\frac{1}{\sqrt{t}})$. 
Result of Figure \ref{Step-size} illustrates that the $\mathtt{GES}(\lambda)$ with a proper constant step-size converges significantly faster than the learning with step-size $\mathcal{O}(\frac{1}{\sqrt{t}})$, which also supports Theorem \ref{linear-convergence}: learning with a proper constant step-size can reach a very faster rate.

\textbf{Comparison of Empirical MSPBE and MSE}
In this section, we use empirical $\text{MSPBE}=\frac{1}{2}\|\hat{b}+\hat{A}\theta\|^{2}_{\hat{M}^{-1}}$ to evaluate the performance, where we evaluate $\hat{A}$, $\hat{b}$, and $\hat{M}$ according to their unbiased estimators by Monte Carlo method with 5000 episodes, and our implementation of MSPBE is inspired by \cite{touati2018convergent}.
Besides, we also compare the performance over a common measurement empirical $\text{MSE}$:
$
\text{MSE}=\|\Phi\theta-q^{\pi}\|^{2}_{\Xi}/\|q^{\pi}\|^{2}_{\Xi}
$, where $q^{\pi}$ is estimated by simulating the target policy $\pi$ and
averaging the discounted cumulative rewards over trajectories.
The combination of step-size for MSE is the same as previous empirical MSPBE.

Results in Figure \ref{MC-ex-result} to \ref{two-state-mpd} show that our $\mathtt{GES}(\lambda)$ learns faster with a better performance than $\mathtt{GQ}(\lambda)$, $\mathtt{ABQ}(\zeta)$ and $\mathtt{GTB}(\lambda)$. 
Besides, $\mathtt{GES}(\lambda)$ converges with a lower variance.
In the Two-State MDP and Baird Star experiments, $\mathtt{GES}(\lambda)$ outperforms the baselines slightly.
This is because both Two-State MDP and Baird Star are relatively easy; many gradient TD learning could learn a convergent result.
While the advantage of $\mathtt{GES}(\lambda)$ over baselines becomes more significant in the MountainCar domain, which shows that $\mathtt{GES}(\lambda)$ is more robust than baselines in the more difficult task.

\begin{figure}[t!]
    \centering
    \subfigure
    {\includegraphics[width=4cm,height=3cm]{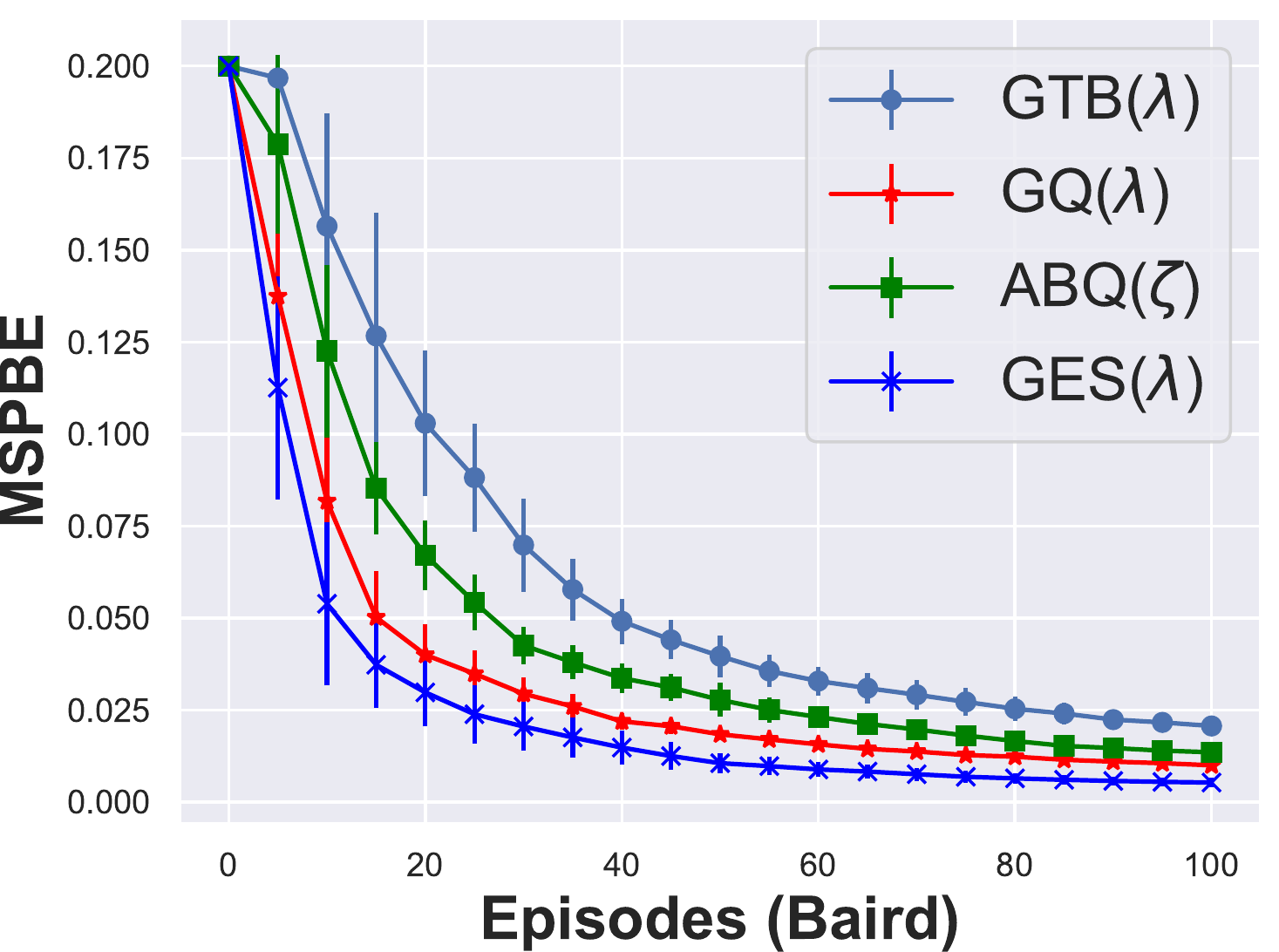}}
    \subfigure
    {\includegraphics[width=4cm,height=3cm]{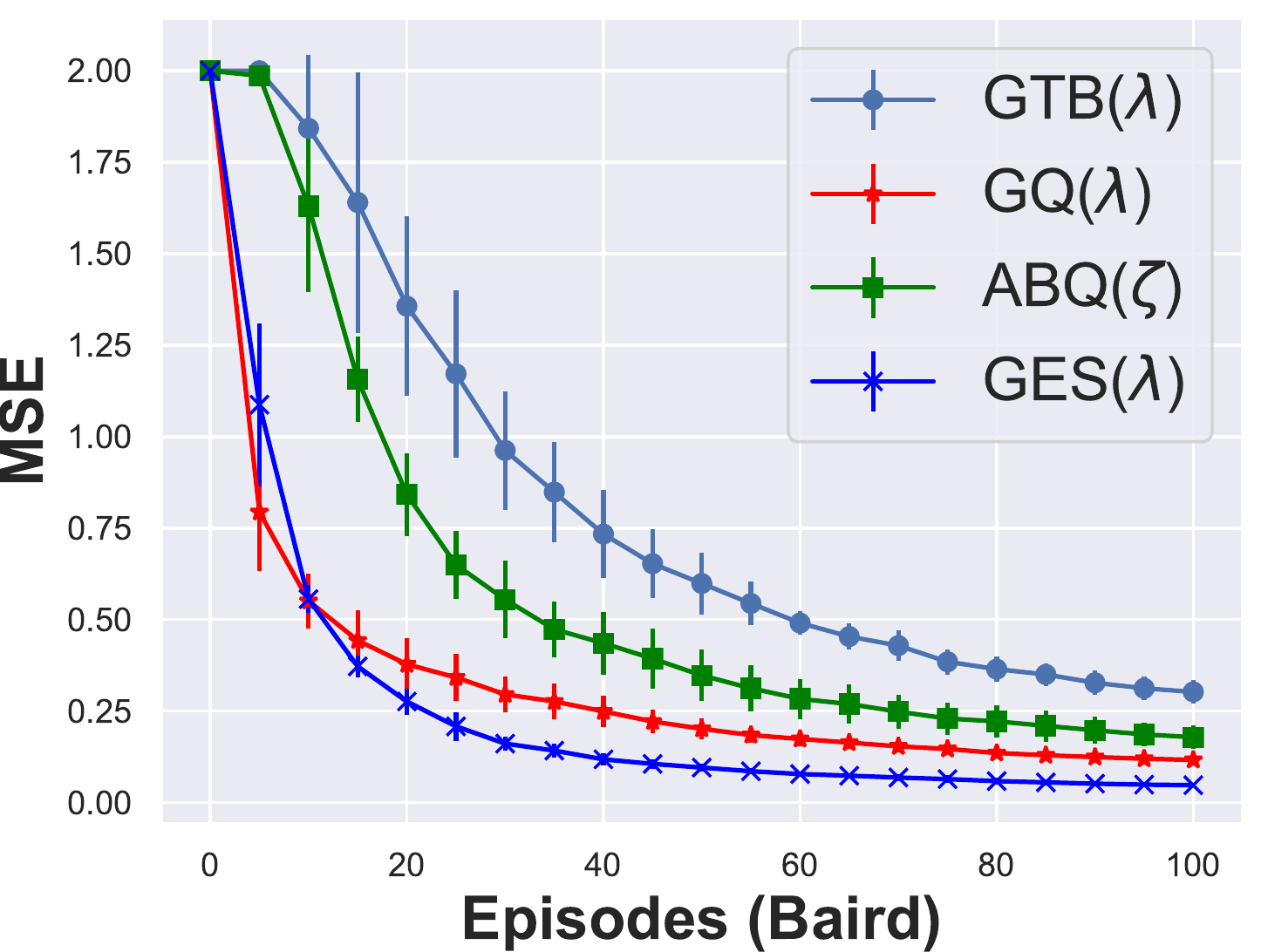}}
    \caption
    {
        MSPBE and MSE comparison on Baird Star.
    }
    \label{ex-baird-mspbe}
\end{figure}
\begin{figure}[t!]
    \centering
    \subfigure
    {\includegraphics[width=4cm,height=3cm]{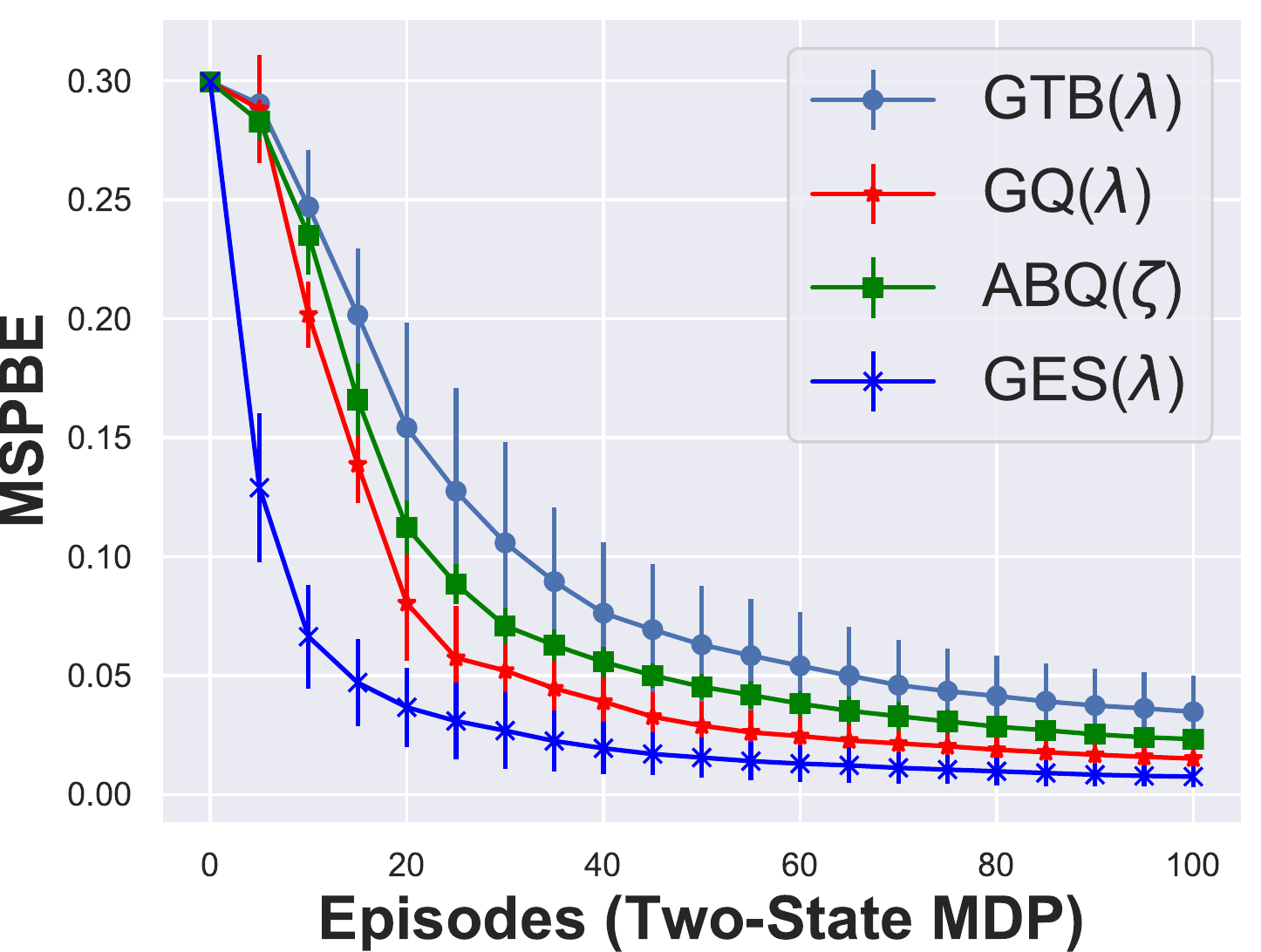}}
    \subfigure
    {\includegraphics[width=4cm,height=3cm]{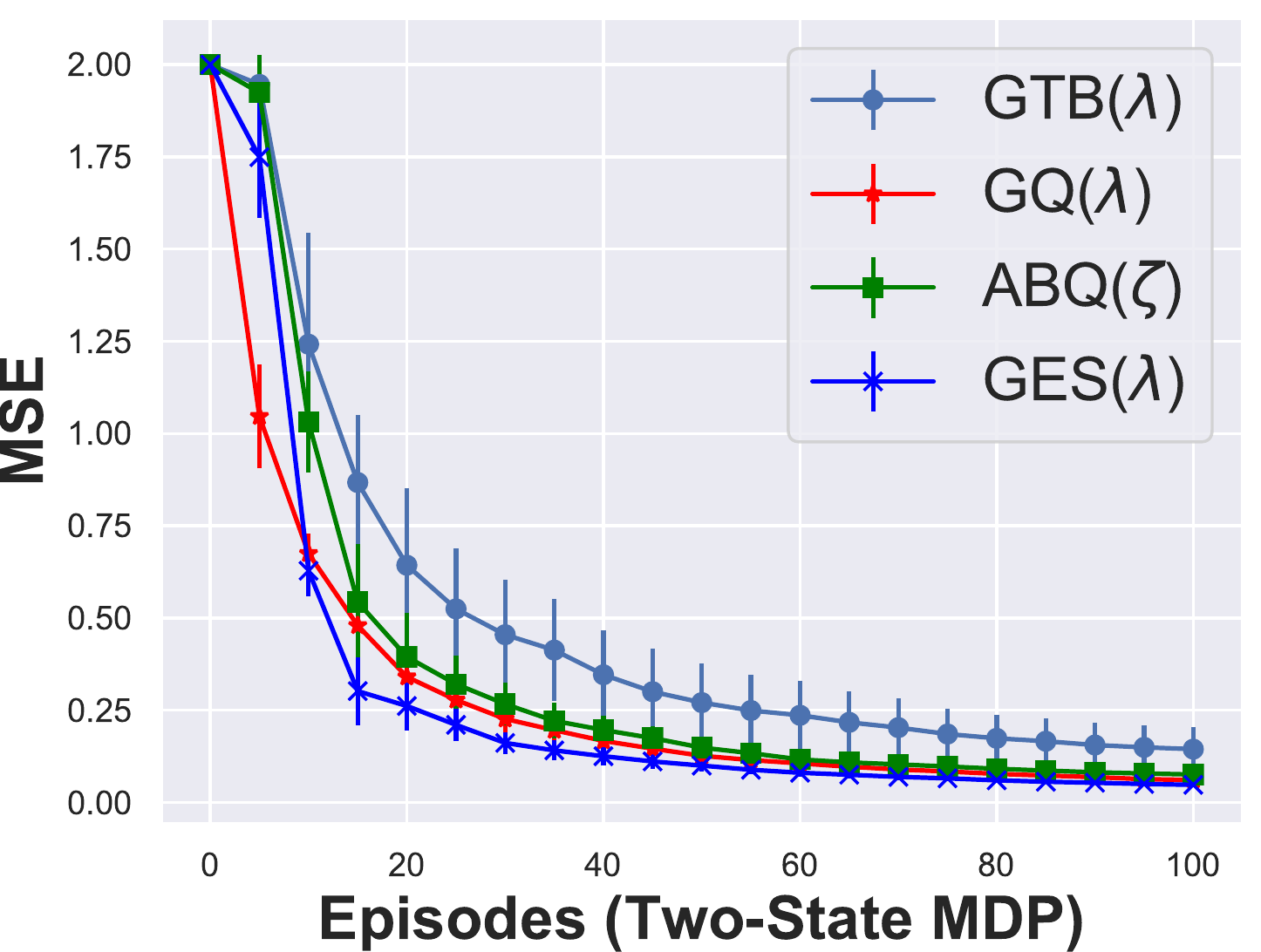}}
    \caption
    {
        MSPBE and MSE comparison on Two-State MDP.
    }
     \label{two-state-mpd}
\end{figure}

\section{Conclusion}
We propose $\mathtt{GES}(\lambda)$ that extends $\mathtt{Expected~Sarsa}(\lambda)$ with linear function approximation.
We prove $\mathtt{GES}(\lambda)$ learns the optimal solution at a linear convergence rate, which is comparable to extensive GTD algorithms.
The primal-dual gap error of $\mathtt{GES}(\lambda)$ matches the best-known theoretical results, but we require a simpler condition of step-size.
 Finally, we conduct experiments to verify the effectiveness of $\mathtt{GES}(\lambda)$.

\clearpage
\bibliographystyle{aaai}
\bibliography{reference}


\clearpage
\onecolumn
\appendix

\clearpage
\section{Appendix A: Importance Sampling and $\lambda$-Operator}

For the discussion of off-policy learning, we need the background of importance sampling. Thus, the basic common conclusion about importance sampling (IS) and pre-decision importance sampling (PDIS) \cite{precup2000eligibility} is necessary.
\subsection{A.1. Off-Policy Learning via Importance Sampling}
\label{app:prop1-1}

Usually, we require that every action taken by $\pi$ is also taken by $\mu$,
which is often called \emph{coverage}~\cite{sutton2018reinforcement} in reinforcement learning.
\begin{assumption}[Coverage]
    \label{ass:Coverage}
    $\forall~(s,a)\in \mathcal{S}\times\mathcal{A}$, we require that
    $\pi(a|s) > 0 \Rightarrow \mu(a|s) > 0$.
\end{assumption}
The difficulty of off-policy roots in the discrepancy between target policy $\pi$ and behavior policy $\mu$ 
----we want to learn the target policy while we only get the data generated by behavior policy. 
One technique to hand this discrepancy is \emph{importance sampling} (IS)~\cite{rubinstein2016simulation}.
Let $\tau_{t}^{h}=\{S_{t},A_{t},R_{t+1}\}_{t\ge0}^{h}$ be a trajectory with finite horizon $h<\infty$.
Let $\rho_{t:k}=\prod_{i=t}^{k}\rho_{i}$ denote the
\emph{cumulated importance sampling ratio}, 
where $\rho_{i}=\frac{\pi(A_{i}|S_{i})}{\mu(A_{i}|S_{i})}$ and $k\leq h$.
Let $G_{t}^{h}=\sum_{k=0}^{h-t-1}\gamma^{k}R_{k+t+1}$, under Assumption~\ref{ass:Coverage} the IS estimator $G_{t}^{\text{IS}}=\rho_{t:h-1}G^{h}_{t}$ is a unbiased estimation of $q^{\pi}$.
However, it is known that IS estimator suffers from large variance of the product $\rho_{t:h-1}$ \cite{sutton1998reinforcement}.
Pre-decision importance sampling (PDIS) \cite{precup2000eligibility} $G_{t}^{\text{PDIS}}=\sum_{k=0}^{h-t-1}\gamma^{k}\rho_{t:t+k}R_{t+k+1}$ is a practical variance reduction method without introducing bias, i.e.
$\mathbb{E}_{\mu}[G_{t}^{\text{PDIS}}|S_{t}=s,A_{t}=a]=q^{\pi}(s,a)$.

\begin{flalign}
\nonumber
\mathbb{E}_{\mu}[\rho_{t:h-1}G_{t}^{h}]
=&
\mathbb{E}_{\mu}[\underbrace{\rho_{t:h-1}R_{t+1}+\rho_{t:h-1}\gamma R_{t+2}+\cdots+\rho_{t:h-1}\gamma^{h-t-1}R_{h}}_{\overset{\text{def}}=G_{t}^{\text{IS}}~~\text{IS-return}}]\\
\nonumber
=&\mathbb{E}_{\mu}[\underbrace{\rho_{t}R_{t+1}+\rho_{t:t+1}\gamma R_{t+2}+\cdots+\rho_{t:h-1}\gamma^{h-t-1}R_{h}}_{\overset{\text{def}}=G_{t}^{\text{PDIS}}~~\text{PDIS-return}}]
=\mathbb{E}_{\mu}[\sum_{k=0}^{h-t-1}\gamma^{k}\rho_{t:t+k}R_{t+k+1}].
\end{flalign}
For the equation $\mathbb{E}_{\mu}[G_{t}^{\text{IS}}]=\mathbb{E}_{\mu}[G_{t}^{\text{PDIS}}]$, please see\cite{precup2000eligibility} or section 5.9 in \cite{sutton2018reinforcement}.
\begin{lemma}[Section 3.10, \cite{thomas2015safe}; Section 5.9,~\cite{sutton2018reinforcement}]
\label{app_A_lemma1}
Let $\tau_{t}^{h}=\{S_{k},A_{k},R_{k+1}\}_{k=t}^{h}$ be the trajectory generated by behavior policy $\mu$, for a given policy $\pi$ and under Assumption~\ref{ass:Coverage}, the following holds,
\begin{flalign}
\label{PDIS-Reward}
\mathbb{E}_{\mu}[\rho_{t:h-1}R_{t+k}]=\mathbb{E}_{\mu}[\rho_{t:t+k-1}R_{t+k}].
\end{flalign}
\end{lemma}
Lemma \ref{app_A_lemma1} implies that for any time $t+k~(k\ge0)$, 
the importance sampling factors after $t+k$ have no effect in the expectation, thus the following holds: for all $k\ge0$,
\begin{flalign}
\label{app1_lemma_corr}
\mathbb{E}_{\mu}[\rho_{t:h-1}R_{t+k}]=\mathbb{E}_{\mu}[\rho_{t:t+k-1}R_{t+k}]=\mathbb{E}_{\pi}[R_{t+k}].
\end{flalign}

\subsection{A.2. Derivation of Eq.(\ref{operator-es})}

\begin{proof}
    \begin{flalign}
    \nonumber
        q+\mathbb{E}_{\mu}[\sum_{k=t}^{\infty}(\lambda\gamma)^{k-t}\delta^{\text{ES}}_{k}\rho_{t+1:k}]\overset{(\ref{app1_lemma_corr})}=&q+\mathbb{E}_{\pi}[\sum_{k=t}^{\infty}(\lambda\gamma)^{k-t}\delta^{\text{ES}}_{k}]\\
        \label{app-7}
        =&q+(I-\lambda\gamma P^{\pi})^{-1}(\mathcal{B}^{\pi}q-q),
    \end{flalign}
    Eq. (\ref{app-7}) is a common result in RL, for the details of \[\mathbb{E}_{\pi}[\sum_{k=t}^{\infty}(\lambda\gamma)^{l-t}\delta^{\text{ES}}_{k}]=(I-\lambda\gamma P^{\pi})^{-1}(\mathcal{B}^{\pi}q-q),\] please refer to  \cite{geist2014off} or Section 6.3.9 in \cite{bertsekas2012dynamic}.
\end{proof}

\section{Appendix B: Proof of Theorem~\ref{positive-convergence}}

\textbf{Theorem~\ref{positive-convergence}}(Stability Criteria)
\emph{
Under Assumption \ref{ass:ergodicity}, the off-line update (\ref{expected-para-es}) is stable if and only if the eigenvalues of the matrix A (\ref{def:A}) have negative real components, i.e.,
\begin{flalign}
\label{condition-stability-app}
\emph{Spec}(A)\subset\mathbb{C}_{-}.
\end{flalign}
}
Before we present the details of its proof, we need some notations of the matrix.
Recall $\text{Spec}(A)$ are the eigenvalues of the matrix $A\in\mathbb{C}^{p\times p}$, we use $\rho(A)$ to denote its spectral radius of the matrix $A$, i.e.,
\[
\rho(A)=\sup_{\lambda\in\text{Spec}(A)}\{|\lambda|\}.
\]

\begin{proof}
Recall 
$
A=\Phi^{\top}\Xi(I-\gamma\lambda P^{\pi})^{-1}(\gamma P^{\pi}-I)\Phi,
$
and 
$\theta^{\star}$ satisfies 
\[
A\theta^{\star}+b=0,
\]
which implies
\begin{flalign}
\nonumber
\theta_{t+1}-\theta^{\star}=\theta_{t}-\alpha(A\theta_t+b)-\theta^{\star}=(I+\alpha A)(\theta_t - \theta^{\star}).
\end{flalign}
Applying above result of recurrently, we have
\begin{flalign}
\label{app-b-condition}
\theta_{t}-\theta^{\star}=(I+\alpha A)^{t}(\theta_0-\theta^{\star}),
\end{flalign}
which implies if the iteration (\ref{app-b-condition}) converges, if and only if \[\rho(I+\alpha A)<1.\]

Furthermore, if the iteration (\ref{app-b-condition}) is stable, if and only if $\rho(I+\alpha A)<1$. If there is a step-size $\alpha>0$ such that $\rho(I+\alpha A)<1$, 
then we have $\text{Spec}(A)\subset\mathbb{C}_{-}$.
Conversely, if $\text{Spec}(A)\subset\mathbb{C}_{-}$, and let \[\alpha=\inf_{\lambda\in\text{Spec}(A)}\Big\{\dfrac{\text{Re}(\lambda)}{|\lambda|}\Big \},\]
then $\rho(I+\alpha A)<1$.
\end{proof}

\section{Appendix C}

\subsection{C.1: Proof of Eq.(\ref{Eq:mspbe})}

For a given policy $\pi$, $Q_{\theta}=\Phi\theta$, then by the definition of MSPBE objection function, we have,
\begin{flalign}
\nonumber
\text{MSPBE}(\theta,\lambda)&=\|Q	_{\theta}-\Pi\mathcal{B}_{\lambda}^{\pi}Q	_{\theta}\|^{2}_{\Xi}\\
\nonumber
&=\|\Pi Q	_{\theta}-\Pi\mathcal{B}_{\lambda}^{\pi}Q	_{\theta}\|^{2}_{\Xi}\\
\nonumber
&=\|\Phi^{T}\Xi(Q	_{\theta}-\mathcal{B}_{\lambda}^{\pi}Q	_{\theta})\|^{2}_{({\Phi^{T}\Xi\Phi})^{-1}}\\
\nonumber
&=\|\Phi^{T}\Xi(I-\lambda\gamma P^{\pi})^{-1}(\Phi\theta-\gamma P^{\pi}\Phi\theta-R^{\pi})\|^{2}_{({\Phi^{T}\Xi\Phi})^{-1}}\\
\nonumber
&=\|\Phi^{T}\Xi(I-\lambda\gamma P^{\pi})^{-1}\big((I-\gamma P^{\pi})\Phi\theta-R^{\pi}\big)\|^{2}_{({\Phi^{T}\Xi\Phi})^{-1}}\\
&=\|b+A\theta\|^{2}_{({\Phi^{T}\Xi\Phi})^{-1}},
\end{flalign}
where $A=\Phi^{T}\Xi(I-\lambda\gamma P^{\pi})^{-1}(\gamma P^{\pi}-I)\Phi,b =\Phi\Xi (I-\lambda\gamma P^{\pi})^{-1}r.$ 

\subsection{C.2: Proof of Proposition \ref{propo-1}}
\textbf{Proposition \ref{propo-1}}
If $(\theta_{\star},\omega_{\star})$ is the solution of the problem (\ref{saddle-point-problem}), then $\theta_{\star}$ is the solution of original problem (\ref{Eq:mspbe}), i.e., \[\theta_{\star}=\arg\min_{\theta}\text{MSPBE}(\theta,\lambda).\]
\begin{proof}
If $\omega_{\star}=\arg\max_{\omega}{(A\theta+b)^{\top}\omega-\frac{1}{2}\|\omega\|_{M}^{2}}$, then $\omega_{\star}=M^{-1}(A\theta+b)$.
Taking $\omega_{\star}$ into (\ref{saddle-point-problem}), then (\ref{saddle-point-problem}) is reduced to $\min_{\theta}\frac{1}{2}\|A\theta+b\|^{2}_{M^{-1}}$.
Now, let $\theta_{\star}$ be the solution of (\ref{saddle-point-problem}), then we have \[\theta_{\star}=\min_{\theta}\frac{1}{2}\|A\theta+b\|^{2}_{M^{-1}}=\arg\min_{\theta}\text{MSPBE}(\theta,\lambda).\]
\end{proof}

\subsection{C.3: Proof of Theorem \ref{linear-convergence}}
\textbf{Theorem \ref{linear-convergence}}
Let $\{(\theta_{t},\omega_{t})\}_{t=0}^{T}$ be generated by Algorithm\ref{alg:algorithm1}.
Recall $g(\omega)$ has been defined in (\ref{Psi}) and $\theta^{\star}$ is the optimal solution of (\ref{Eq:mspbe}).
Let $\nu=\dfrac{2\kappa^{2}(A)\kappa(M)\sigma_{\max}(A)}{\sigma_{\min}(M)}$,
$\Delta_{\theta_{t}}=\|\theta_{t}-\theta_{\star}\|^{2}_{2}$, $\Delta_{\omega_{t}}=\|\omega_{t}-\nabla g^{\star}(A\theta_{t})\|^{2}_{2}$, and $D_{t}=\nu \Delta_{\theta_{t}}+ \Delta_{\omega_{t}}$.
Under Assumption \ref{ass:ergodicity}, and we assume $\emph{rank}(\Phi)=p$. Let
\[
\alpha=\dfrac{\sigma_{\min}(M)}{\big(\sigma_{\max}(M)+\sigma_{\min}(M)\big)\big(\frac{\sigma^{2}_{\max}(A)}{\sigma_{\min}(M)}+\nu\sigma_{\max}(A)\big)},
\beta=\dfrac{2}{\sigma_{\max}(M)+\sigma_{\min}(M)},
\]
under Assumption \ref{ass:ergodicity}, we have
\begin{flalign}
\label{P_t}
\mathbb{E}[D_{t+1}]\leq \Big(1-\dfrac{1}{12}\dfrac{1}{\kappa^{3}(M)\kappa^{4}(A)}\Big) \mathbb{E}[D_{t}].
\end{flalign}
Furthermore, we have
\[
\mathbb{E}[\|\theta_{t}-\theta_{\star}\|^{2}_{2}]\leq \dfrac{1}{\nu}\Big(1-\dfrac{1}{12}\dfrac{1}{\kappa^{3}(M)\kappa^{4}(A)}\Big)^{t}\mathbb{E}[D_{0}].
\]
\begin{proof}
The proof is inspired by a general analysis that appears in \cite{du2019linear}; we refer the reader to that reference for further technical details.

Firstly, we apply Theorem 3.1 of \cite{du2019linear} to achieve (\ref{P_t}), which requires us to check $g(\omega)$ is a Lipschitz-smooth and strongly convex function.
Recall $g(\omega)=\frac{1}{2}\|\omega\|^{2}_{M}-b^{\top}\omega$, for any $\omega_1$, $\omega_2$, we have
\begin{flalign}
\label{app-c-01}
\|\nabla g(\omega_1)-\nabla g(\omega_1)\|_2=\|M(\omega_1-\omega_2)\|_{2}
\leq\|M\|_{\text{op}}\|\omega_1-\omega_2\|_2
\overset{(\textbf{a})}=\sigma_{\max}(M)\|\omega_1-\omega_2\|_2,
\end{flalign}
where the last equation (\textbf{a}) of (\ref{app-c-01}) holds since $M=\mathbb{E}_{\mu}[\phi_{t}\phi_{t}^{\top}]=\Phi^{\top}\Xi\Phi$ is a positive symmetric matrix if $\text{rank}(\Phi)=p$, then \[\|M\|_{\text{op}}=\max\{\lambda_1,\lambda_2,\cdots,\lambda_p\}=\sigma_{\max}(M).\]
Eq.(\ref{app-c-01}) implies $g(\omega)$ is $\sigma_{\max}(M)$-smooth function.
Under Assumption \ref{ass:ergodicity}, $\Xi\succ 0$; recall rank$(\Phi)$$=$$p$, so $M=\Phi^{\top} \Xi\Phi\succ 0$,
then $\nabla^{2}g(\omega)=M$ implies \[\nabla^{2}g(\omega)-\sigma_{\min}(M) I\succ0,\] thus $g(\omega)$ is a $\sigma_{\min}(M)$-strongly convex function.
Let step-size satisfy the following condition
\[
\alpha=\dfrac{\sigma_{\min}(M)}{\big(\sigma_{\max}(M)+\sigma_{\min}(M)\big)\big(\frac{\sigma^{2}_{\max}(A)}{\sigma_{\min}(M)}+\nu\sigma_{\max}(A)\big)},
\beta=\dfrac{2}{\sigma_{\max}(M)+\sigma_{\min}(M)},
\]
then according to Theorem 3.1 of \cite{du2019linear}, we have
\[\mathbb{E}[D_{t+1}]\leq \Big(1-\dfrac{1}{12}\dfrac{1}{\kappa^{3}(M)\kappa^{4}(A)}\Big) \mathbb{E}[D_{t}].\]
Furthermore, the fact $D_{t}=\nu \Delta_{\theta_{t}}+ \Delta_{\omega_{t}}\ge\nu \Delta_{\theta_{t}}$ implies
\[
\mathbb{E}[\|\theta_{t}-\theta_{\star}\|^{2}_{2}]\leq\frac{\mathbb{E}[D_{t}]}{\nu}\overset{(\ref{P_t})}\leq \dfrac{1}{\nu}\Big(1-\frac{1}{12}\frac{1}{\kappa^{3}(M)\kappa^{4}(A)}\Big)^{t}\mathbb{E}[D_{0}].
\]
Therefore the proof is completed.
\end{proof}

\begin{corollary}
Consider all the conditions and notations in Theorem \ref{linear-convergence}, 
the output of Algorithm \ref{alg:algorithm1} closes to $(\theta^{\star},\omega^{\star})$ as follows,
\[
\mathbb{E}[\|\theta_t-\theta^{\star}\|_2^{2}]\leq\delta^2,~~~\mathbb{E}[\|\omega_t-\omega^{\star}\|_2^{2}]\leq\delta^2,
\]
if after a computational cost of
\[
\mathcal{O}
\bigg(
\max\Big\{1,\frac{\lambda_{\max}(A)}{\lambda_{\max}(M)\nu}\Big\}
\Big(1-\frac{1}{12\kappa^{3}(M)\kappa^{4}(A)}\Big)\log(\frac{1}{\delta})
\bigg).
\]
\end{corollary}

\begin{proof}
Furthermore, recall $\omega^{\star}=(\nabla g)^{-1}(A\theta^{\star})=\nabla g^{\star}(A\theta^{\star})$, then we have
\begin{flalign}
\nonumber
\|\omega_t-\omega^{\star}\|_2
\leq&\|\omega_t-\nabla g^{\star}(A\theta_t)\|_2+\|\nabla g^{\star}(A\theta_t)-\omega^{\star}\|_2\\
\nonumber
\leq&\Delta_{\omega_t}+\|\nabla g^{\star}(A\theta_t)-g^{\star}(A\theta^{\star})\|_2\\
\nonumber
\leq&\Delta_{\omega_t}+\dfrac{\lambda_{\max}(A)}{\lambda_{\max}(M)}\Delta_{\theta_{t}}\leq\max\Big\{1,\dfrac{\lambda_{\max}(A)}{\lambda_{\max}(M)\nu}\Big\}D_t.
\end{flalign}
Since $\mathbb{E}[D_{t+1}]\leq \Big(1-\frac{1}{12}\frac{1}{\kappa^{3}(M)\kappa^{4}(A)}\Big) \mathbb{E}[D_{t}]$, then we have
\[
\mathbb{E}[\|\omega_t-\omega^{\star}\|_2]\leq\max\Big\{1,\dfrac{\lambda_{\max}(A)}{\lambda_{\max}(M)\nu}\Big\}
\mathbb{E}[D_{t+1}]\leq \max\Big\{1,\dfrac{\lambda_{\max}(A)}{\lambda_{\max}(M)\nu}\Big\}\Big(1-\dfrac{1}{12}\dfrac{1}{\kappa^{3}(M)\kappa^{4}(A)}\Big)^{t} \mathbb{E}[D_{0}].
\]
If, $\mathbb{E}[\|\omega_t-\omega^{\star}\|_2]\leq\delta$, then the times of update at least
\[
\mathcal{O}
\bigg(
\max\Big\{1,\frac{\lambda_{\max}(A)}{\lambda_{\max}(M)\nu}\Big\}
\Big(1-\frac{1}{12\kappa^{3}(M)\kappa^{4}(A)}\Big)\log(\frac{1}{\delta})
\bigg).
\]
This conclude the result of Remark \ref{remark-sec-com}.
\end{proof}

\section{Appendix D: Proof of Theorem \ref{step-size-per}}

\subsection{D.1: Some Preliminaries}

\textbf{Assumption \ref{ass:Boundedness-of-Parameter}} (Boundedness of Feature Map, Reward)
\emph{
    The features $\{\phi_{t}\}_{t\ge0}$ is uniformly bounded by $\phi_{\max}$. 
    The reward function is uniformly bounded by $R_{\max}$.
    The importance sampling $\rho_{t}$ is uniformly bounded by $\rho_{\max}$.
}

\begin{remark}
\label{app-reark-01}
Assumption \ref{ass:Boundedness-of-Parameter} implies the boundedness of $\hat{A}_{t}$, $\hat{M}_{t}$ and $\hat{b}_{t}$.
In fact, 
\[
\|\hat{M}_{t}\|_{op}^{2}=\|\phi_{t}\phi^{\top}_{t}\|_{op}^{2}\leq (p\sqrt{p}\phi^2_{\max})^2=:C^2_{M},
\]
where we use a basic fact for matrix: for each $A\in\mathbb{R}^{p\times p}$, we have $\|A\|_{op}\leq p\sqrt{p}\|A\|_{\infty}$.
The boundedness of $\hat{M}_{t}$ also imply the boundedness of $M$, i..e,
\[
\|M\|_{op}\leq C_{M}.
\]
Furthermore, since $M=\Phi^{\top}\Xi\Phi$ and $\text{rank}(\Phi)=p$, thus, both $M$ and $M^{-1}$ are all positive define matrix,
then, there exists a positive scalar ${C}_{M^{-1}}$ such that 
\[
\|M^{-1}\|_{op}\leq C_{M^{-1}}.
\]
Recall $e_{t}=\lambda\gamma \rho_{t}e_{t-1}+\phi_{t}$, then
\[
\|e_{t}\|_{2}^{2}=\big\|\sum_{k=0}^{t}(\gamma\lambda)^{t-k}\rho_{k+1:t}\phi_{k}\big\|_{2}^{2}\leq\Big(\frac{\phi_{\max}}{1-\gamma\lambda\rho_{\max}}\Big)^2=:C^{2}_{e}.
\]
Since then, 
\begin{flalign}
\nonumber
\|\hat{b}_{t}\|_{2}^{2}=&\|R_{t+1}e_{t}\|_{2}^{2}\leq R^{2}_{\max}C^{2}_{e}=:C^{2}_{b},\\
\nonumber
\|\hat{A}_{t}\|_{2}^{2}=&\|e_{t}(\gamma\mathbb{E}_{\pi}[\phi(S_{t+1,\cdot})]-\phi_{t})^{\top}\|^{2}_{2}
\leq (\gamma+1)^{2}C^{2}_{e}\phi^{2}_{\max}=:C^{2}_{A}.
\end{flalign}
Above results also imply the boundedness of $A$ and $b$, i.e.,
\[
\|A\|_{op}\leq C_A,~~\text{and}~~\|b\|_{2}\leq C_b.
\]
\end{remark}

\begin{lemma}[\cite{gupta2019finite}]
\label{app-d-lemma-matr}
Consider the matrix $A$ as follows
\[
A(x)=\dfrac{1}{a_1+a_2}
\begin{pmatrix}
a_2 &-a_1a_2\\
-a_1a_2&\dfrac{1}{x}a_1-xy a_1
\end{pmatrix},
\]
where $a_1,a_2,c,y>0$. Then we have
\[
\lambda_{\min}(A(x))\ge\kappa_1-\kappa_2x,
\]
where
$\kappa_1=\dfrac{a_2}{a_1+a_2}$ and $\kappa_2$ is a constant that only depends on $a_1,a_2$ and $x$.
\end{lemma}

\subsection{D.2: Proof of Lemma \ref{lyapunov-function-improve-one-step}}

\textbf{Lemma \ref{lyapunov-function-improve-one-step}}
\emph{
Under Assumption \ref{ass:ergodicity}-\ref{ass:Solvability-of-Problem}, there exists a positive scalar $\tau$ such that: $t\ge\tau$,
\begin{flalign}
\nonumber
\mathbb{E}[L(z_{t+1})]-\mathbb{E}[L(z_{t})]
&\leq-\alpha\Big(\frac{1}{2}\varkappa_1
-\frac{\alpha}{\beta}\varkappa_2\Big)
\mathbb{E}[L(z_t)]+2\alpha^2\zeta^{2}\lambda_{\max}(Q)\widetilde{c_{b}}^2,
\end{flalign}
where $\varkappa_1,\varkappa_2,\zeta,\widetilde{c_{b}}$ are constants that we will concretize them later.
}

We define
\begin{flalign}
\nonumber
\hat{H}_t&=-\hat{A}_{t}^{\top}M^{-1}{A}, ~~~~H=-{A}^{\top}M^{-1}{A},\\
\nonumber
\hat{L}_{t}&=\hat{A}_{t}+\hat{M}_{t}M^{-1}A,~~~~L=2A.
\end{flalign}
By the boundedness of the previous term, we have
\begin{flalign}
\nonumber
\|\hat{H}_t\|_{op}&\leq\|\hat{A}_{t}^{\top}M^{-1}{A}\|_{op}\leq C^{2}_{A}C_{M^{-1}}, \\
\nonumber
\|\hat{L}_{t}\|_{op}&\leq\|\hat{A}_{t}\|_{op}+\|\hat{M}_{t}M^{-1}A\|_{op}\leq C_A+C_M C_{M^{-1}} C_A, ~\|L\|_{op}\leq 2C_A.
\end{flalign}
Furthermore, let $Q_1$ and $Q_2$ be the solutions of the following equations
\begin{flalign}
\label{PQ-EQ-app}
 \begin{cases}
    H^{\top}Q_1+Q_1H&=-I\\
    M^{\top}Q_2+Q_2M&=I,
 \end{cases}
\end{flalign}
since both $M$ and $L$ are Hurwitz matrix, then solution of (\ref{PQ-EQ-app}) always exists.
Furthermore, we define a matrix $P$ as follows,
\begin{flalign}
\label{def:mat-P}
Q=
\begin{pmatrix}
\dfrac{\|Q_1A^{\top}\|_{op}}{\|Q_1A^{\top}\|_{op}+\|Q_2M^{-1}AL\|_{op}}Q_1
&0\\
0&\dfrac{\|Q_2M^{-1}AL\|_{op}}{\|Q_1A^{\top}\|_{op}+\|Q_2M^{-1}AL\|_{op}}Q_2
\end{pmatrix}.
\end{flalign}
We define 
\begin{flalign}
\label{online-revisit}
\varrho_{t}&=\omega_t-M^{-1}A\theta_{t}, ~~~
z_t=
\begin{pmatrix}
\theta_{t}-\theta^{\star}\\
\varrho_{t}-\varrho^{\star},
\end{pmatrix}.
\end{flalign}
where $\varrho_{\star}=\omega^{\star}-M^{-1}A\theta^{\star}$.
Let $L(z_t)$ be the following Lyapunov function,
\begin{flalign}
\label{lyapunov-function}
L(z_t)=z_t ^{\top}Q z_t.
\end{flalign}

\begin{proof}
Recall the update (\ref{stochastic-im-1}) and (\ref{stochastic-im}):
\[
\omega_{t+1}=\omega_{t}+\beta(\hat A_{t}{\theta}_{t}+\hat b_{t}-\hat M_{t}\omega_{t}) ,~~
\theta_{t+1}=\theta_{t}-\alpha \hat A_{t}^{\top}\omega_{t},
\]
we can rewrite $z_t$ (\ref{online-revisit}) as the following recursion:
\begin{flalign}
\label{z-t-online-revisit}
z_{t+1}=z_t+\alpha(\tilde{G}_t z_t+\tilde{g}_t),
\end{flalign}
where 
\begin{flalign}
\tilde{G}_t=
\begin{pmatrix}
\hat{H}_t,&-\hat{A}^{\top}_{t}\\
-M^{-1}A\hat{H}_{t}+\dfrac{\beta}{\alpha}\hat{L}_{t}&M^{-1}A\hat{A}^{\top}_{t}-\dfrac{\beta}{\alpha}\hat{M}_t
\end{pmatrix},
\tilde{g}_t=
\begin{pmatrix}
0\\
\dfrac{\beta}{\alpha}\hat{b}_t
\end{pmatrix}.
\end{flalign}
Let $\tilde{G}_{\infty}=:\lim_{t\rightarrow\infty}\tilde{G}_t$ and $\tilde{g}_{\infty}=:\lim_{t\rightarrow\infty}\tilde{g}_t$.  Furthermore, we define
\begin{flalign}
\nonumber
{G}_\infty=\lim_{t\rightarrow\infty}\mathbb{E}[\tilde{G}_t]=
\begin{pmatrix}
{H},&{A}^{\top}\\
-M^{-1}A H+\dfrac{\beta}{\alpha}L&M^{-1}A{A}^{\top}-\dfrac{\beta}{\alpha}M
\end{pmatrix},
{g}_\infty=
\begin{pmatrix}
0\\
-\dfrac{\beta}{\alpha}b
\end{pmatrix}.
\end{flalign}
\underline{Boundedness of $\tilde{G}_t$, $\tilde{G}_{\infty}$, $\tilde{g}_t$, $\tilde{g}_{\infty}$.}
\begin{flalign}
\nonumber
\|\tilde{G}_t\|_{op}\leq\|\hat{H}_t\|_{op}+\|\hat{A}_t\|_{op}+\|M^{-1}A\hat{H}_{t}\|_{op}+\Big\|\dfrac{\beta}{\alpha}\hat{L}_{t}\Big\|_{op}
+\|M^{-1}A\hat{A}_{t}\|_{op}+\Big\|\dfrac{\beta}{\alpha}\hat{M}_{t}\Big\|_{op}.
\end{flalign}
Recall the boundedness of $\hat{H}_t,\hat{A}_t,M^{-1},A,\hat{L}_{t}$ and $\hat{M}_{t}$, then the following holds
\begin{flalign}
\label{bound-tda-A}
\|\tilde{G}_t\|_{op}&\leq
\underbrace{2C^{2}_{A}C_{M^{-1}}
+C_A+C^{3}_{A}C^{2}_{M^{-1}}}_{=:C_1}
+\dfrac{\beta}{\alpha}
\Big(\underbrace{C_A+C_M C_{M^{-1}} C_A+C_{M}}_{=:C_2}\Big)
=C_1+\dfrac{\beta}{\alpha} C_2=:\zeta.
\end{flalign}
Furthermore, we have
\[
\|\tilde{G}_{\infty}\|_{op}\leq
\zeta,
\|\tilde{g}_t\|_{2}\leq\dfrac{\beta}{\alpha} c_b=C_2\dfrac{\beta}{\alpha}\dfrac{c_b}{C_2}\leq\dfrac{c_b}{C_2}\zeta=:\widetilde{c_b}\zeta,~~\|\tilde{b}_\infty\|_{2}=0.
\]
Recall (\ref{z-t-online-revisit}), we have
\begin{flalign}
\label{z-t-online-revisit-norm}
\|z_{t+1}- z_t\|_{2}\leq \alpha\|\tilde{G}_t z_t+\tilde{g}_t\|_2\leq\alpha\zeta(\|z_t\|_2+\widetilde{c_b}).
\end{flalign}
Furthermore, according to the same analysis of Lemma 3 in \cite{srikant2019-colt-finite-time}, we have
\begin{flalign}
\nonumber
\|z_\tau-z_0\|_2&\leq 2\alpha\zeta\tau \|z_0\|_2+2\alpha\zeta\tau \widetilde{c_b},\\
\nonumber
\|z_\tau-z_0\|_2&\leq 4\alpha\zeta\tau \|z_\tau\|_2+4\alpha\zeta\tau \widetilde{c_b},\\
\nonumber
\|z_\tau-z_0\|^2_2&\leq 32\alpha^2\zeta^2\tau^2 \|z_\tau\|^2_2+32\alpha^2\zeta^2\tau^2 \widetilde{c_{b}}^2.
\end{flalign}
Recall (\ref{z-t-online-revisit-norm}), we have
\begin{flalign}
\nonumber
\big|
(z_{t+1}-z_{t})^{\top}P(z_{t+1}-z_{t})\big|\leq \lambda_{\max}(Q)\|z_{t+1}-z_{t}\|^2_2\leq\lambda_{\max}(Q)(\alpha\zeta(\|z_t\|_2+\widetilde{c_b}))^2\leq
2\alpha^2\zeta^{2}\lambda_{\max}(Q)(\|z_t\|_{2}^{2}+\widetilde{c_{b}}^2),
\end{flalign}
where the last inequality holds since $(a+b)^2\leq 2(a^2+b^2)$.
Let 
\[
X_{t}=:\{S_0,A_0,S_1,A_1,\cdots,S_t,A_t\},
\]
after some careful calculations, we have: for all $t\ge\tau$, 
\begin{flalign}
\nonumber
&\Big|\mathbb{E}
\Big[
z_{t}^{\top}\Big(P{G}_{\infty}z_t-\dfrac{1}{\alpha}(z_{t+1}-z_{t})\Big)\Big|z_{t-\tau}, X_{t-\tau}
\Big]
\Big|\\
=&\Big|\mathbb{E}
\Big[
z_{t}^{\top}\Big(P{G}_{\infty}z_t-\tilde{G}_t z_t\Big)\Big|z_{t-\tau}, X_{t-\tau}
\Big]
\Big|\leq\zeta(1+\lambda_{\max}(Q))
\mathbb{E}[\|z_{t}\|^2_2|z_{t-\tau}, X_{t-\tau}]
\label{app-d-bound-01}
\end{flalign}

For $t\ge \tau$, we have
\begin{flalign}
\nonumber
&\mathbb{E}
[L(z_{t+1})-L(z_{t})|z_{t-\tau}, X_{t-\tau}]\\
\nonumber
=&\mathbb{E}
[z_{t+1} ^{\top}P z_{t+1}-z_t ^{\top}P z_t|z_{t-\tau}, X_{t-\tau}]\\
\nonumber
=&
\mathbb{E}
[(z_{t+1}-z_{t})^{\top}P(z_{t+1}-z_{t})+2z_{t}^{\top}P(z_{t+1}-z_{t})
|z_{t-\tau}, X_{t-\tau}]\\
\nonumber
=&
\mathbb{E}
[\underbrace{(z_{t+1}-z_{t})^{\top}P(z_{t+1}-z_{t})}_{\leq 2\alpha^2\zeta^{2}\lambda_{\max}(Q)(\|z_t\|_{2}^{2}+\widetilde{c_{b}}^2)}+
2z_{t}^{\top}P(z_{t+1}-z_{t}-\alpha G_{\infty} z_t)+2\alpha z_{t}^{\top}PG_{\infty}z_t
|z_{t-\tau}, X_{t-\tau}]\\
\label{app-eq-02}
\leq&
\mathbb{E}
[2\alpha^2\zeta^{2}\lambda_{\max}(Q)(\|z_t\|_{2}^{2}+\widetilde{c_{b}}^2)+
2z_{t}^{\top}P(z_{t+1}-z_{t}-\alpha G_{\infty} z_t)+2\alpha z_{t}^{\top}PG_{\infty} z_t
|z_{t-\tau}, X_{t-\tau}].
\end{flalign}

Furthermore, since
$
\mathbb{E}[z_{t}^{\top}PG_\infty z_t
|z_{t-\tau}, X_{t-\tau}
]
\leq -\lambda_{\min}(\Sigma)\mathbb{E}[\|z_t\|_{2}^{2}|z_{t-\tau}, X_{t-\tau}],
$
where $\lambda_{\min}(\Sigma)$ is the smallest eigenvalue of the the following matrix $\Sigma$:
\[
\Sigma=
\begin{pmatrix}
\dfrac{\xi_2}{\xi_1+\xi_2}&-\dfrac{\xi_1\xi_2}{\xi_1+\xi_2}\\
-\dfrac{\xi_1\xi_2}{\xi_1+\xi_2}&\dfrac{\xi_1\Big(\frac{\beta}{\alpha}-2\|Q_2A^{\top}M^{-1}A\|_{op}\Big)}{\xi_1+\xi_2}
\end{pmatrix},
\]
and $\xi_1=2\|Q_1A^{\top}\|_{op}$, and $\xi_2=2\|Q_2M^{-1}AL\|_{op}$. 
Recall the result of Lemma \ref{app-d-lemma-matr}, we have
\begin{flalign}
\label{app-d-02}
\lambda_{\min}(\Sigma)\ge\dfrac{\xi_2}{\xi_1+\xi_2}-\dfrac{\alpha}{\beta}\varkappa_1=:\varkappa_1-\dfrac{\alpha}{\beta}\varkappa_2,
\end{flalign}
where $\varkappa_1=\dfrac{\xi_2}{\xi_1+\xi_2}=\dfrac{\|Q_1A^{\top}\|_{op}}{\|Q_2M^{-1}AL\|_{op}+\|Q_1A^{\top}\|_{op}}$, $\varkappa_2$ is a constant depends on $\xi_1,\xi_2$ and $\dfrac{\alpha}{\beta}$.

From the result of (\ref{app-d-bound-01}) and (\ref{app-eq-02}), we have
\begin{flalign}
\nonumber
&\mathbb{E}[L(z_{t+1})-L(z_{t})|z_{t-\tau}, X_{t-\tau}]\\
\nonumber
\leq&2\alpha^2\zeta^{2}\lambda_{\max}(Q)\mathbb{E}[\|z_t\|_{2}^{2}|z_{t-\tau}, X_{t-\tau}]
+2\alpha^2\zeta^{2}\lambda_{\max}(Q)\widetilde{c_{b}}^2
+2\alpha \mathbb{E}[z_{t}^{\top}PG_{\infty} z_t
|z_{t-\tau}, X_{t-\tau}]\\
\nonumber
&+2\underbrace{\mathbb{E}[z_{t}^{\top}P(z_{t+1}-z_{t}-\alpha G_{\infty} z_t)|z_{t-\tau}, X_{t-\tau}]}_{=2\alpha\mathbb{E}\big[z_{t}^{\top}P\big(\dfrac{z_{t+1}-z_{t}}{\alpha}- G_{\infty} z_t\big)\big|z_{t-\tau}, X_{t-\tau}\big]}\\
\nonumber
\overset{(\ref{app-d-bound-01})}\leq&
2\alpha^2\zeta^{2}\lambda_{\max}(Q)\mathbb{E}[\|z_t\|_{2}^{2}|z_{t-\tau}, X_{t-\tau}]+2\alpha^2\zeta^{2}\lambda_{\max}(Q)\widetilde{c_{b}}^2
+2\alpha \mathbb{E}[z_{t}^{\top}PG_{\infty} z_t
|z_{t-\tau}, X_{t-\tau}]\\
\nonumber
&+2\alpha\zeta(1+\lambda_{\max}(Q))\mathbb{E}[\|z_t\|^2_2|z_{t-\tau}, X_{t-\tau}]\\
\nonumber
\overset{(\ref{app-d-02})}\leq&
2\alpha^2\zeta^{2}\lambda_{\max}(Q)\mathbb{E}[\|z_t\|_{2}^{2}|z_{t-\tau}, X_{t-\tau}]+2\alpha^2\zeta^{2}\lambda_{\max}(Q)\widetilde{c_{b}}^2
-2\alpha\lambda_{\min}(\Sigma) \mathbb{E}[\|z_{t}\|_2^2
|z_{t-\tau}, X_{t-\tau}]\\
\nonumber
&+2\alpha\zeta(1+\lambda_{\max}(Q))\mathbb{E}[\|z_t\|^2_2|z_{t-\tau}, X_{t-\tau}]\\
\nonumber
\leq&
2\alpha^2\zeta^{2}\lambda_{\max}(Q)\mathbb{E}[\|z_t\|_{2}^{2}|z_{t-\tau}, X_{t-\tau}]+2\alpha^2\zeta^{2}\lambda_{\max}(Q)\widetilde{c_{b}}^2
-2\alpha\Big(\varkappa_1-\dfrac{\alpha}{\beta}\varkappa_2\Big)\mathbb{E}[\|z_{t}\|_2^2
|z_{t-\tau}, X_{t-\tau}]\\
\nonumber
&+2\alpha\zeta(1+\lambda_{\max}(Q))\mathbb{E}[\|z_t\|^2_2|z_{t-\tau}, X_{t-\tau}]\\
\nonumber
=&
\Big(-2\alpha\big(\varkappa_1-\dfrac{\alpha}{\beta}\varkappa_2\big)+2\alpha\zeta(1+\lambda_{\max}(Q))+2\alpha^2\zeta^{2}\lambda_{\max}(Q)\Big)\mathbb{E}[\|z_{t}\|_2^2
|z_{t-\tau}, X_{t-\tau}]
+
2\alpha^2\zeta^{2}\lambda_{\max}(Q)\widetilde{c_{b}}^2\\
\label{last-01}
\leq&
\Big(-\alpha\big(\varkappa_1-\dfrac{\alpha}{\beta}\varkappa_2\big)\Big)\mathbb{E}[\|z_{t}\|_2^2
|z_{t-\tau}, X_{t-\tau}]
+
2\alpha^2\zeta^{2}\lambda_{\max}(Q)\widetilde{c_{b}}^2
\end{flalign}
where the last (\ref{last-01}) holds since we need an additional condition as follows,
\[
\Big(-\alpha\big(\varkappa_1-\dfrac{\alpha}{\beta}\varkappa_2\big)+\alpha\zeta(1+\lambda_{\max}(Q))+\alpha^2\zeta^{2}\lambda_{\max}(Q)\Big)\leq0.
\]
\end{proof}

\subsection{D.3: Proof of Theorem \ref{step-size-per}}

\begin{proof}
Lemma \ref{lyapunov-function-improve-one-step} implies
\begin{flalign}
\nonumber
\mathbb{E}[L(z_{t+1})]-\mathbb{E}[L(z_{t})]
&\leq-\alpha\Big(\frac{1}{2}\varkappa_1-\frac{\alpha}{\beta}\varkappa_2\Big)\mathbb{E}[\|z_{t}\|_2^2]+2\alpha^2\zeta^{2}\lambda_{\max}(Q)\widetilde{c_{b}}^2\\
\nonumber
&\leq-\alpha\Big(\frac{1}{2}\varkappa_1-\frac{\alpha}{\beta}\varkappa_2\Big)\dfrac{1}{\lambda_{\max}(Q)}\mathbb{E}[L(z_{t})]
+2\alpha^2\zeta^{2}\lambda_{\max}(Q)\widetilde{c_{b}}^2.
\end{flalign}
Rewrite above equation, we have
\[
\mathbb{E}[L(z_{t+1})]
\leq\bigg(1-\alpha\Big(\frac{1}{2}\varkappa_1-\frac{\alpha}{\beta}\varkappa_2\Big)\dfrac{1}{\lambda_{\max}(Q)}\bigg)\mathbb{E}[L(z_{t})]
+2\alpha^2\zeta^{2}\lambda_{\max}(Q)\widetilde{c_{b}}^2,
\]
i.e., we have
\[
\mathbb{E}[L(z_{t+1})]-\dfrac{2\alpha^2\zeta^{2}\lambda_{\max}(Q)\widetilde{c_{b}}^2}{\alpha\Big(\frac{1}{2}\varkappa_1-\frac{\alpha}{\beta}\varkappa_2\Big)\dfrac{1}{\lambda_{\max}(Q)}}
\leq\bigg(1-\alpha\Big(\frac{1}{2}\varkappa_1-\frac{\alpha}{\beta}\varkappa_2\Big)\dfrac{1}{\lambda_{\max}(Q)}\bigg)
\Bigg(
\mathbb{E}[L(z_{t})]
-\dfrac{2\alpha^2\zeta^{2}\lambda_{\max}(Q)\widetilde{c_{b}}^2}{\alpha\Big(\frac{1}{2}\varkappa_1-\frac{\alpha}{\beta}\varkappa_2\Big)\dfrac{1}{\lambda_{\max}(Q)}}
\Bigg).
\]
Rewrite above equation, we have
\[
\mathbb{E}[L(z_{t+1})]-\dfrac{2\alpha^2\zeta^{2}\lambda^2_{\max}(P)\widetilde{c_{b}}^2}{\alpha\Big(\frac{1}{2}\varkappa_1-\frac{\alpha}{\beta}\varkappa_2\Big)}
\leq\bigg(1-\alpha\Big(\frac{1}{2}\varkappa_1-\frac{\alpha}{\beta}\varkappa_2\Big)\dfrac{1}{\lambda_{\max}(Q)}\bigg)
\Bigg(
\mathbb{E}[L(z_{t})]
-\dfrac{2\alpha^2\zeta^{2}\lambda^2_{\max}(P)\widetilde{c_{b}}^2}{\alpha\Big(\frac{1}{2}\varkappa_1-\frac{\alpha}{\beta}\varkappa_2\Big)}
\Bigg).
\]
To simplify notations, we introduce 
\[
u=1-\alpha\Big(\frac{1}{2}\varkappa_1-\frac{\alpha}{\beta}\varkappa_2\Big)\dfrac{1}{\lambda_{\max}(Q)}, 
v={2\alpha^2\zeta^{2}\lambda_{\max}(Q)\widetilde{c_{b}}^2}.
\] 
Furthermore, applying above equation recursively, we have
\[
\mathbb{E}[L(z_{t})]\leq u^{t-\tau}\mathbb{E}[L(z_{\tau})]+v\dfrac{1-u^{t-\tau}}{1-u}\leq u^{t-\tau}\mathbb{E}[L(z_{\tau})]+\dfrac{v}{1-u}.
\]
Then, the following equation holds
\begin{flalign}
\label{app-bound-001}
\mathbb{E}[\|z_t\|_2^{2}]\leq\dfrac{1}{\lambda_{\min}(P)}\mathbb{E}[L(z_{t})]\leq\dfrac{1}{\lambda_{\min}(P)}\Big(
u^{t-\tau}\mathbb{E}[L(z_{\tau})]+\dfrac{v}{1-u}\Big).
\end{flalign}
Additionally,
\begin{flalign}
\nonumber
\mathbb{E}[L(z_{\tau})]\leq\lambda_{\max}(Q)\mathbb{E}[\|z_{\tau}\|_2^2]&\leq\lambda_{\max}(Q)(\mathbb{E}[\|z_{\tau}-z_0\|_2^2]+\|z_0\|_2^2)
\\
\nonumber
&\leq\lambda_{\max}(Q)((2\alpha\zeta\tau \|z_0\|_2+2\alpha\zeta\tau \widetilde{c_b})^2+\|z_0\|_2^2)\\
\nonumber
&=\lambda_{\max}(Q)(4\alpha^2\zeta^2\tau^2(\|z_0\|_2+\widetilde{c_b})^2+\|z_0\|_2^2).
\end{flalign}
Taking it to (\ref{app-bound-001}), we have
\begin{flalign}
\mathbb{E}[\|z_t\|_2^{2}]&\leq\dfrac{\lambda_{\max}(Q)}{\lambda_{\min}(P)}
u^{t-\tau}(4\alpha^2\zeta^2\tau^2(\|z_0\|_2+\widetilde{c_b})^2+\|z_0\|_2^2)+\dfrac{1}{\lambda_{\min}(P)}\dfrac{v}{1-u}\\
\nonumber
&=\alpha^2\zeta^{2}u^{t-\tau}\underbrace{\kappa(P)(4\zeta^2\tau^2(\|z_0\|_2+\widetilde{c_b})^2+\|z_0\|_2^2)}_{=:\eta_1}+\alpha
\underbrace{\kappa(P)\dfrac{2\lambda_{\max}(Q)\widetilde{c_{b}}^2}{\Big(\frac{1}{2}\varkappa_1-\frac{\alpha}{\beta}\varkappa_2\Big)}}_{=:\eta_2}
\\
\nonumber
&=\alpha^2\eta_1\bigg(1-\dfrac{\alpha}{\lambda_{\max}(Q)}\Big(\frac{1}{2}\varkappa_1-\frac{\alpha}{\beta}\varkappa_2\Big)\bigg)
^{t-\tau}+\alpha\eta_2.
\end{flalign}
\end{proof}

\section{Appendix E: Performance over Primal-Dual Gap Error}
According to Nemirovski et al., \shortcite{nemirovski2009robust}, we can measure the convergence of problem (\ref{saddle-point-problem}) by \emph{primal-dual gap error}.

\begin{definition}[Primal-Dual Gap Error]
Recall $\Psi$ defined in (\ref{Psi}), the primal-dual gap error $\epsilon_{\Psi}(\theta,\omega)$ at each solution $(\omega,\theta)$ is defined as: 
    \[
    \epsilon_{\Psi}(\theta,\omega)=\max_{\omega^{'}} \Psi(\theta,\omega^{'})- \min_{\theta^{'}} \Psi(\theta^{'},\omega).
    \]
\end{definition}

\begin{theorem}[Convergence of Algorithm \ref{alg:algorithm1}]
    \label{theo:on-algo2-convergence}
    Under Assumption \ref{ass:ergodicity}-\ref{ass:Boundedness-of-Parameter}.
    Consider the sequence $\{(\theta_{t},\omega_{t})\}_{t=1}^{T}$ generated Algorithm \ref{alg:algorithm1}.
    Let $C$ be a constant defined in (\ref{def:C}), step-size $\alpha_{t}=\beta_{t}=\dfrac{2}{C\sqrt{5t}}$
    and $\tilde{\theta}_{T}=\dfrac{\sum_{t=1}^{T}\alpha_{t}\theta_{t}}{\sum_{t=1}^{T}\alpha_{t}}$, $\tilde{\omega}_{T}=\dfrac{\sum_{t=1}^{T}\alpha_{t}\omega_{t}}{\sum_{t=1}^{T}\alpha_{t}}$.
    Then primal-dual gap error
    $\epsilon_{\Psi}(\tilde{\theta}_{T},\tilde{\omega}_{T})$ is upper-bounded by
    \begin{flalign}
    \label{PD-gap-1}
    \mathbb{E}[\epsilon_{\Psi}(\tilde{\theta}_{T},\tilde{\omega}_{T})]\leq C{\sqrt{\dfrac{5}{T}}}.
    \end{flalign}
Furthermore, for any $\delta\in(0,\frac{2}{e})$, the following holds with probability at least $1-\delta$, 
\begin{flalign}
\label{prob-gap-1}
\epsilon_{\Psi}(\tilde{\theta}_{T},\tilde{\omega}_{T})\leq C{\sqrt{\dfrac{5}{T}}\Big(8+2\log\dfrac{2}{\delta}\Big)}.
\end{flalign}
\end{theorem}
\begin{proof}
The proof of Theorem \ref{theo:on-algo2-convergence} relies on some results in the section 3.1 of \cite{nemirovski2009robust}, we refer the reader to that reference for further technical details.
Let $\widehat{G}(\theta,\omega)$ be the stochastic gradient vector of $\Psi(\theta,\omega)$:
\[
\widehat{G}(\theta,\omega)=
 \begin{pmatrix}
\hat{g}_{\theta}(\theta,\omega) \\
\hat{g}_{\omega}(\theta,\omega)
\end{pmatrix}
=
 \begin{pmatrix}
\hat{A}_{t}\omega \\
\hat{A}_{t}\theta+\hat{b}_{t}-\hat{M}_{t}\omega
\end{pmatrix}.
\]
According to Nemirovski et al., \shortcite{nemirovski2009robust}, we need to check: 
I) $\widehat{G}(\theta,\omega)$ is an unbiased estimate of the gradient of $\Psi({\theta,\omega})$;
II) $\mathbb{E}[\|\widehat{G}(\theta,\omega)\|]$ is uniformly bounded on the region $D_{\theta}\times D_{\omega}$.

From (\ref{unbiased-A-b-M}), $\hat{g}_{\theta}(\theta,\omega)$ and $\hat{g}_{\omega}(\theta,\omega)$ are the unbiased estimates of $\partial_{\theta}\Psi({\theta,\omega})$ and  $\partial_{\omega}\Psi({\theta,\omega})$ correspondingly:
\begin{flalign}
\mathbb{E}[\widehat{G}(\theta,\omega)]=
 \begin{pmatrix}
\mathbb{E}[\hat{g}_{\theta}(\theta,\omega)] \\
\mathbb{E}[\hat{g}_{\omega}(\theta,\omega)]
\end{pmatrix}
= \begin{pmatrix}
\partial_{\theta}\Psi({\theta,\omega}) \\
\partial_{\omega}\Psi({\theta,\omega})
\end{pmatrix}.
\end{flalign}
Furthermore, we should check for each $(\theta_{k},\omega_{k})$, the terms $\mathbb{E}[\hat{g}_{\theta}(\theta_{k},\omega_{k})]$ and $\mathbb{E}[\hat{g}_{\omega}(\theta_k,\omega_k)]$ are uniformly bounded:
\begin{flalign}
\nonumber
\mathbb{E}[\|\hat{g}_{\omega}(\theta_{k},\omega_{k})\|_{2}^{2}]&=\mathbb{E}[\|\hat A_{t}{\theta}_{t}+\hat b_{t}-\hat M_{t}\omega_{t}\|_{2}^{2}]
\leq C^{2}_{b}+C^{2}_{A}\text{diam}^{2}(D_{\theta})+C^{2}_{M}\text{diam}^{2}(D_{\omega})
\overset{\text{def}}=\widetilde{C}^{2}_1,\\
\nonumber
\mathbb{E}[\|\hat{g}_{\theta}(\theta_{k},\omega_{k})\|_{2}^{2}]&=\mathbb{E}[\|\hat{A}_{t}\omega\|_{2}^{2}]\leq  C^{2}_{A}\text{diam}^{2}(D_{\omega})\overset{\text{def}}=\widetilde{C}^{2}_2,
\end{flalign}
where ``diam'' is short for diameter.
Let $C$ be a constant:
\begin{flalign}
\label{def:C}
C=4\text{diam}^{2}(D_{\omega})\widetilde{C}^{2}_1+\text{diam}^{2}(D_{\theta})\widetilde{C}^{2}_2.
\end{flalign}
Then, according to Eq.(3.15) in \cite{nemirovski2009robust}, the result (\ref{PD-gap-1}) holds.
Furthermore, by the Proposition 3.2 of \cite{nemirovski2009robust}, for any $\eta>1$:
\[\mathbb{P}\Big[\epsilon_{\Psi}(\tilde{\theta}_{T},\tilde{\omega}_{T})>C\dfrac{8+2\eta}{\sqrt{T}}\Big]\leq 2e^{-\eta},\]
which implies for any $\delta\in(0,\frac{2}{e})$, the result (\ref{prob-gap-1}) holds with probability at least $1-\delta$.
\end{proof}

\begin{discussion}
[Comparison with Existing Works over Primal-Dual Gap Error]
Liu et al. \shortcite{liu2015finite} firstly derive $\mathtt{GTD}$ via convex-concave saddle-point formulation, and their optimal convergence rate reaches $\mathbb{E}[\epsilon_{\Psi}(\tilde{\theta}_{T},\tilde{\omega}_{T})]=\mathcal{O}({1}/{\sqrt{T}})$.
Later, Wang et al.\shortcite{wang2017finite} extends the work of Liu et al.\shortcite{liu2015finite}, they suppose the data is generated from Markov processes rather than I.I.D assumption. Wang et al.\shortcite{wang2017finite} prove the convergence rate $\mathbb{E}[\epsilon_{\Psi}(\tilde{\theta}_{T},\tilde{\omega}_{T})]=\mathcal{O}\bigg(\dfrac{\sum_{t=1}^{T}\alpha^{2}_{t}}{\sum_{t=1}^{T}\alpha_{t}}\bigg)$, the optimal convergence rate also reaches $\mathcal{O}({1}/{\sqrt{T}})$, where they require step-size satisfies $\sum_{t=1}^{\infty}\alpha_t=\infty$, $\dfrac{\sum_{t=1}^{T}\alpha^{2}_{t}}{\sum_{t=1}^{T}\alpha_{t}}\leq\infty$.
Our work, i.e., Theorem \ref{theo:on-algo2-convergence} matches the best-known theoretical results Wang et al.\shortcite{wang2017finite}, while we point out a simpler and concrete step-size.
\end{discussion}

\section{Appendix F: Additional Details of Experiments}

For the limitation of space, in this section, we present all the details of experiments.

\textbf{MountainCar}
Since the state space of mountaincar domain is continuous, we use the open tile coding software 
\url{http://incompleteideas.net/rlai.cs.ualberta.ca/RLAI/RLtoolkit/tilecoding.html} to extract feature of states.

In this experiment, we set the number of tilings to be 4 and there are no white noise features. 
The performance is an average 5 runs and each run contains 5000 episodes. 
We set $\lambda=0.99$, $\gamma=0.99$.  
The MSPBE/MSE distribution is computed over the combination of step-size,
$(\alpha_{t},\frac{\beta_{t}}{\alpha_{t}})\in[0.1\times 2^{j}|j = -10,-9,\cdots,-1, 0]^{2}$, and $\lambda=0.99$.
Following suggestions from Section10.1 in \cite{sutton2018reinforcement}, we set all the initial state-action values to be  0, which is optimistic to cause extensive exploration.

\textbf{Baird Example} The Baird example considers the episodic seven-state, two-action MDP.
The $\mathtt{dashed}$ action takes the system to one of the six upper states with equal probability, whereas the $\mathtt{solid}$ action takes the system to the seventh state. 
The behavior policy $b$ selects the $\mathtt{dashed}$ and $\mathtt{solid}$ actions with probabilities $\frac{6}{7}$ and $\frac{1}{7}$, so that the next-state distribution under it is uniform (the
same for all nonterminal states), which is also the starting distribution for each episode. The target policy $\pi$ always takes the solid action, and so the on-policy distribution (for $\pi$) is concentrated in the seventh state. The reward is zero on all transitions. The discount rate is  $\gamma=0.99$.
The feature $\phi(\cdot,{\mathtt{dashed}})$ and $\phi(\cdot,{\mathtt{solid}})$ are defined as follows,
\begin{flalign}
\nonumber
\phi(\mathtt{s_1},{\mathtt{dashed}})& = 
(2,0,0,0,0,0,0,1,0,0,0,0,,0,,0,0,0,0 )\\
\nonumber
\phi(\mathtt{s_2},{\mathtt{dashed}}) &= 
(0,2,0,0,0,0,0,1,0,0,0,0,,0,,0,0,0,0 )\\
\nonumber
&\cdots\\
\phi(\mathtt{s_7},{\mathtt{dashed}}) &= 
(0,0,0,0,0,0,2,1,0,0,0,0,,0,,0,0,0,0 ),
\end{flalign}
\begin{flalign}
\nonumber
\phi(\mathtt{s_1},{\mathtt{solid}})& = 
(0,0,0,0,0,0,0,0,2,0,0,0,,0,,0,0,0,1 )\\
\nonumber
\phi(\mathtt{s_2},{\mathtt{solid}}) &= 
(0,0,0,0,0,0,0,0,0,2,0,0,,0,,0,0,0,1 )\\
\nonumber
&\cdots\\
\phi(\mathtt{s_7},{\mathtt{solid}}) &= 
(0,0,0,0,0,0,0,0,0,0,0,0,,0,,0,0,2,1 ).
\end{flalign}

\end{document}